\documentclass{article}
\usepackage[utf8]{inputenc}

  \PassOptionsToPackage{square,comma,numbers}{natbib}

\usepackage[preprint]{nips_2018}

\title{Differentially Private Contextual Linear Bandits}
\author{%
  Roshan Shariff \\
  Department of Computing Science\\
  University of Alberta\\
  Edmonton, Alberta, Canada \\
  \texttt{roshan.shariff@ualberta.ca}
  \And{}
  Or Sheffet \\
  Department of Computing Science\\
  University of Alberta\\
  Edmonton, Alberta, Canada \\
  \texttt{osheffet@ualberta.ca}
}

\usepackage{newtxtext}
\usepackage[T1]{fontenc} 
\usepackage{amsmath,mathtools}
\usepackage{amsthm,thmtools,thm-restate}
\usepackage{amssymb} 
\usepackage{nicefrac}
\usepackage{newtxmath}
\usepackage[scr=rsfso]{mathalfa} 
\usepackage{dsfont}
\usepackage{bm} 
\usepackage{etoolbox}

\usepackage{lineno}
\usepackage{microtype}
\usepackage[inline,shortlabels]{enumitem}
\usepackage{setspace}
\usepackage{tikz-cd}
\usepackage{booktabs}
\usepackage{algpseudocode,algorithm}
\usepackage{multicol}
\usepackage{graphicx}
\usepackage{color}
\usepackage{subcaption}

\usepackage{varioref}
\usepackage[pdfusetitle,bookmarks]{hyperref}
\usepackage[capitalise]{cleveref}

\hypersetup{%
  hidelinks,
  bookmarksnumbered,
  bookmarksopen,
  bookmarksopenlevel=2
}

\usepackage[obeyFinal]{todonotes}

\renewcommand{\vec}[1]{\bm{#1}}
\newcommand{\wildcard}{\mathinner{\,{\cdot}\,}}
\newcommand{\defeq}{\coloneq}
\newcommand{\eqdef}{\eqcolon}
\newcommand{\inv}[1]{#1^{-1}}
\newcommand{\Real}{\mathds{R}}

\newcommand{\UCB}{\mathrm{UCB}}

\newcommand{\iid}{\text{i.i.d.\@}}
\renewcommand{\Pr}{\mathds{P}}

\newcommand{\argmin}{\operatorname*{arg\,min}}
\newcommand{\argmax}{\operatorname*{arg\,max}}
\newcommand{\tr}{\operatorname{tr}}
\newcommand{\rank}{\operatorname{rank}}
\renewcommand{\det}{\operatorname{det}}

\newcommand\given[1][\delimsize]{%
  \providecommand{\delimsize}{}
  \nonscript\:#1\vert\allowbreak\nonscript\:\mathopen{}
}

\DeclarePairedDelimiter{\abs}||
\DeclarePairedDelimiter{\paren}{\lparen}{\rparen}
\DeclarePairedDelimiter{\brck}{\lbrack}{\rbrack}
\DeclarePairedDelimiterX{\set}[1]\lbrace\rbrace{#1}
\DeclarePairedDelimiter{\floor}{\lfloor}{\rfloor}
\DeclarePairedDelimiter{\ceil}{\lceil}{\rceil}
\DeclarePairedDelimiterX{\innerp}[2]\langle\rangle{#1,#2}
\DeclarePairedDelimiterXPP{\Prob}[1]{\Pr}{\lparen}{\rparen}{}{#1}
\DeclarePairedDelimiterXPP{\PrSet}[1]{\Pr}{\lbrace}{\rbrace}{}{#1}
\DeclarePairedDelimiterXPP{\Ex}[1]{\mathds{E}}{\lbrack}{\rbrack}{}{#1}
\DeclarePairedDelimiterXPP{\Exx}[2]{\mathds{E}_{#1}}{\lbrack}{\rbrack}{}{#2}
\DeclarePairedDelimiterXPP{\Var}[1]{\mathrm{Var}}{\lbrack}{\rbrack}{}{#1}
\DeclarePairedDelimiterXPP{\One}[1]{\mathds{1}}\{\}{}{#1}
\DeclarePairedDelimiterXPP{\norm}[2]{}\Vert\Vert{_{#1}}{#2}

\newcommand{\A}{\mathcal{A}}
\newcommand{\C}{\mathcal{C}}
\newcommand{\D}{\mathcal{D}}
\newcommand{\E}{\mathcal{E}}
\providecommand\transp{\top}
\let\transpsymbol\transp{}
\renewcommand{\transp}[1]{#1^\transpsymbol}

\newcommand{\Wishart}{\mathcal{W}}
\newcommand{\Normal}{\mathcal{N}}
\newcommand{\Eye}[1][]{\bm{I}\notblank{#1}{_{{#1}\times{#1}}}{}}
\newcommand{\XtX}[1]{\transp{#1}{#1}}

\declaretheorem[style=plain]{theorem}
\declaretheorem[style=plain,sibling=theorem]{lemma}
\declaretheorem[style=plain,sibling=theorem]{corollary}
\declaretheorem[style=plain,sibling=theorem]{proposition}
\declaretheorem[style=plain,sibling=theorem]{claim}
\declaretheorem[style=definition,sibling=theorem]{definition}

\declaretheorem[style=remark]{remark}

\declaretheorem[style=definition,numbered=no]{assumptions}

\newenvironment{assumptions*}[2][]{%
  \begin{assumptions}[#1]
    #2
    \begin{enumerate}[nolistsep]
      \setcounter{enumi}{\theassumption}
      \newcommand{\assume}[1][]{\item\label[assumption]{##1}}
    }{%
      \setcounter{assumption}{\theenumi}
    \end{enumerate}
  \end{assumptions}%
}

\Crefname{assumption}{Assumption}{Assumptions}

\begin{document}
\setlength{\intextsep}{4pt}
\setlength{\textfloatsep}{4pt}
\setlength{\abovedisplayskip}{3pt}
\setlength{\belowdisplayskip}{3pt}

\phantomsection{}
\makeatletter\addcontentsline{toc}{chapter}{\@title}\makeatother
\maketitle

\begin{abstract}
  We study the contextual linear bandit problem, a version of the
  standard stochastic multi-armed bandit (MAB) problem where a learner
  sequentially selects actions to maximize a reward which
  depends also on a user provided per-round \emph{context}. Though the context
  is chosen arbitrarily or adversarially, the reward is assumed to
  be a stochastic function of a feature vector that encodes the
  context and selected action. Our goal is to devise private learners for the
  contextual linear bandit problem.

  We first show that using the standard definition of differential
  privacy results in linear regret. So instead, we adopt the notion of \emph{joint}
  differential privacy, where we assume that the action chosen on day
  $t$ is only revealed to user $t$ and thus needn't be kept private
  that day, only on following days. We give a general scheme
  converting the classic linear-UCB algorithm into a joint differentially
  private algorithm using the tree-based algorithm~\cite{ChanPrivateContinualRelease2010,DworkContinualObservation2010}.
  We then apply either Gaussian noise or Wishart noise to achieve joint-differentially private algorithms and bound the resulting algorithms' regrets. In addition, we give the first lower bound on the
  \emph{additional} regret any private algorithms for the MAB problem must incur.
 \end{abstract}

\section{Introduction}%
\label{sec:introduction}

The well-known \emph{stochastic multi-armed bandit} (MAB) is a
sequential decision-making task in which a learner repeatedly chooses
an action (or arm) and receives a noisy reward.  The objective is to
maximize cumulative reward by \emph{exploring} the actions to discover
optimal ones (having the best expected reward), balanced with
\emph{exploiting} them.  The \emph{contextual} bandit problem is an
extension of the MAB problem, where the learner also receives a \emph{context} in
each round, and the expected reward depends on \emph{both} the context
and the selected action.


As a motivating example, consider online shopping: the user provides a
context (composed of query words, past purchases, etc.), and the
website responds with a suggested product and receives a reward if the
user buys it.  Ignoring the context and modeling the problem as a
standard MAB (with an action for each possible product) suffers from
the drawback of ignoring the variety of users' preferences; whereas
separately learning each user's preferences doesn't allow us to
generalize between users.  Therefore it is common to model the task as
a contextual \emph{linear bandit} problem: Based on the user-given
context, each action is mapped to a feature vector; the reward
probability is then assumed to depend on the \emph{same} unknown
linear function of the feature vector across all users.

The above example motivates the need for privacy in the contextual
bandit setting: users' past purchases and search queries are sensitive personal
information, yet they strongly predict future purchases.
In this work, we give upper and lower bounds for the problem of
(joint) \emph{differentially private} contextual linear
bandits. Differential privacy is the \emph{de facto} gold standard of
privacy-preserving data analysis in both academia and industry,
requiring that an algorithm's output have very limited dependency on
any single user interaction (one context and reward).  However, as we
later illustrate, adhering to the standard notion of differential
privacy (under event-level continual observation) in the contextual
bandit requires us to essentially ignore the context and thus incur
linear regret.  We therefore adopt the more relaxed notion of
\emph{joint differential privacy}~\citep{KearnsMechanismDesign2014}
which, intuitively, allows us to present the $t$-th user with products
corresponding to her preferences, while guaranteeing that all
interactions with all users at times $t' > t$ have very limited
dependence on user $t$’s preferences.  The guarantee of differential
privacy under continuous observation assures us that even if all later
users collude in an effort to learn user $t$'s context or preference,
they still have very limited advantage over a random guess.

\subsection{Problem Formulation}%
\label{subsec:problem_formulation}
\paragraph{Stochastic Contextual Linear Bandits.}
In the classic MAB, in every round $t$ a learner selects an
\emph{action} $a_t$ from a fixed set $\A$ and receives a \emph{reward}
$y_t$.  In the (stationary) \emph{stochastic} MAB, the reward is noisy
with a fixed but unknown expectation $\Ex{y_t\given a_t}$ that depends
only on the selected action.  In the stochastic contextual bandit
problem, before each round the learner also receives a \emph{context}
$c_t\in\C$~--- the expected reward $\Ex{y_t\given c_t,a_t}$ depends on
both $c_t$ and $a_t$.  It is common to assume that the context affects
the reward in a linear way: map every context-action pair to a
\emph{feature vector} $\phi(c,a)\in\Real^d$ (where $\phi$ is an
arbitrary but known function) and assume that
$\Ex{y_t\given c_t,a_t} = \innerp{\vec\theta^*}{\phi(c_t,a_t)}$.  The
vector $\vec\theta^*\in\Real^d$ is the key unknown parameter of the
environment which the learner must discover to maximize reward.
Alternatively, we say that on every round the learner is given a
\emph{decision set} $\D_t \defeq \set{\phi(c_t,a)\given a\in\A}$ of
all the pre-computed feature vectors: choosing $\vec x_t\in\D_t$
effectively determines the action $a_t\in\A$.  Thus, the contextual
stochastic linear bandit framework consists of repeated rounds in
which the learner
\begin{enumerate*}[(i),before=\unskip{: },itemjoin={{; }},itemjoin*={{; and }}]
\item receives a \emph{decision set} $\mathcal{D}_t \subset \Real^d$
\item chooses an \emph{action} $\vec x_t \in \mathcal{D}_t$
\item receives a stochastic \emph{reward}
  $y_t = \innerp{\vec\theta^*}{\vec x_t} + \eta_t$.
\end{enumerate*}
When all the $\D_t$ are identical and consist of the standard basis
vectors, the problem reduces to standard MAB\@.

The learner's objective is to maximize cumulative reward, which is
equivalent to minimizing \emph{regret}: the extra reward a learner
would have received by always choosing the best available action.
In other words, the regret characterizes the cost of having to
\emph{learn} the optimal action over just \emph{knowing} it
beforehand.  For stochastic problems, we are usually interested in a
related quantity called \emph{pseudo-regret}, which is the extra
\emph{expected} reward that the learner could have earned if it had
known $\vec\theta^*$ in advance.  In our setting, the cumulative
pseudo-regret after $n$ rounds is
$\widehat{R}_n \defeq \sum_{t=1}^n\max_{\vec{x}\in\D_t}\innerp{\vec\theta^*}{\vec x - \vec x_t}$.%
\footnote{The pseudo-regret ignores the stochasticity of the reward
  but not the resulting randomness in the learner's choice of actions.
  It equals the regret in expectation, but is more amenable to
  high-probability bounds such as ours.  In particular, in some cases
  we can achieve polylog$(n)$ bounds on pseudo-regret because, unlike
  regret, it doesn't have added regret noise of variance $\Omega(n)$.}

\paragraph{Joint Differential Privacy.} %
As discussed above, the context and reward may be considered private
information about the users which we wish to keep private from all
\emph{other} users. We thus introduce the notion of jointly
differentially private learners under continuous observation, a
combination of two definitions \citep[given
in][]{KearnsMechanismDesign2014,DworkContinualObservation2010}. First,
we say two sequences
$S = \langle (\mathcal{D}_1, y_1), (\mathcal{D}_2, y_2), \dotsc,
(\mathcal{D}_n, y_n) \rangle$ and
$S' = \langle (\mathcal{D}'_1, y'_1), \dotsc, (\mathcal{D}'_n, y'_n)
\rangle$ are \emph{$t$-neighbors} if for all $t'\neq t$ it holds that
$(\mathcal{D}_{t'},y_{t'}) = (\mathcal{D}'_{t'}, y'_{t'})$.

\begin{definition}%
\label{def:joint-dp}
  A randomized algorithm $A$ for the contextual bandit problem is
  \emph{$(\varepsilon,\delta)$-jointly differentially private} (JDP) under continual observation if for any $t$ and any pair of $t$-neighboring sequences $S$ and $S'$, and any subset ${\cal S}_{>t} \subset \mathcal{D}_{t+1} \times \mathcal{D}_{t+2} \times \dotsb \times \mathcal{D}_{n}$ of sequence of actions ranging from day $t+1$ to the end of the sequence, it holds that $\Prob{A(S)\in \mathcal{S}_{>t}} \leq e^\varepsilon\Prob{A(S')\in \mathcal{S}'_{>t}} +\delta$.
\end{definition}
The standard notion of differential privacy under continual
observation would require that changing the context $c_t$ cannot have
much effect on the probability of choosing action $a_t$ --- even for
round $t$ itself (not just for future rounds as with JDP)\@.  In our
problem formulation, however, changing $c_t$ to $c_t'$ may change the
decision set $\D_t$ to a possibly disjoint $\D_t'$, making that notion
ill-defined.  Therefore, when we discuss the impossibility of
regret-minimization under standard differential privacy in
\cref{sec:lower-bounds}, we revert back to a fixed action set $\A$
with an explicit per-round context $c_t$.


\subsection{Our Contributions and Paper Organization}%
\label{subsec:contributions}

In this work, in addition to formulating the definition of JDP under
continual observation, we also present a framework for implementing
JDP algorithms for the contextual linear bandit problem. Not
surprisingly, our framework combines a tree-based privacy
algorithm~\citep{ChanPrivateContinualRelease2010,DworkContinualObservation2010}
with a linear upper confidence bound (LinUCB)
algorithm~\cite{DaniStochasticLinearOptimization2008}.  For
modularity, in \cref{sec:LinUCB} we analyze a family of linear
UCB algorithms that use different regularizers in every round, under
the premise that the regularizers are PSD with bounded singular
values.  Moreover, we repeat our analysis twice --- first we obtain a
general $\tilde O(\sqrt n)$ upper bound on regret; then, for problem
instances that maintain a $\Delta$ reward gap separating the optimal
and sub-optimal actions, we obtain a $\mathrm{polylog}(n)/\Delta$
regret upper bound.
Our leading application of course is privacy, though one could
postulate other reasons where such changing regularizers would be
useful (e.g., if parameter estimates turn out to be wrong and have to
be updated).  We then plug two particular regularizers into our
scheme: the first is a privacy-preserving mechanism that uses additive
Wishart noise~\citep{SheffetPrivateApproxRegression2015} (which is
always PSD); the second uses additive Gaussian
noise~\citep{DworkAnalyzeGauss2014} (shifted to make it PSD w.h.p.\@
over all rounds). The main term in the two regret bounds obtained by
both algorithms is $\tilde O(\sqrt n\cdot d^{3/4}/\sqrt\varepsilon)$
(the bound itself depends on numerous parameters, a notation-list
given in \cref{sec:preliminaries}).  Details of both
techniques appear in \cref{sec:alg-dp}.  Experiments with a few
variants of our algorithms are detailed in \cref{sec:experiments} of
the supplementary material.  In \cref{sec:lower-bounds} we also give a
lower bound for the $\varepsilon$-differentially private MAB
problem. Whereas all previous work on the private MAB problem uses
standard (non-private) bounds, we show that \emph{any} private
algorithm must incur \emph{an additional} regret of
$\Omega(k\log(n)/\varepsilon)$. While the result resembles the lower
bound in the adversarial setting, the proof technique cannot rely on
standard packing arguments~\citep[e.g.][]{HardtTalwarGeometryDP2010}
since the input for the problem is stochastic rather than adversarial.
Instead, we rely on a recent coupling
argument~\citep{KarwaVadhanFiniteSampleDP2017} to prove any private
algorithm must substantially explore suboptimal arms.

\paragraph{Future Directions.} The linear UCB algorithm we adapt in
this work is a canonical approach to the linear bandit problem, using
the principle of ``optimism in the face of uncertainty.''  However,
recent work~\citep{LattimoreS17} shows that all such ``optimistic''
algorithms are sub-optimal, and instead proposes adapting to the
decision set in a particular way by solving an intricate optimization
problem.  It remains an open question to devise a private version of
this algorithm which interpolates between UCB and fine-tuning to the
specific action set.


\subsection{Related Work}%
\label{subsec:related_work}

\paragraph{MAB and the Contextual Bandit Problem.}
The MAB dates to the seminal work of \citet{robbins1952}, with the UCB
approach developed in a series of works~\citep{BanditBook85,Agrawal95}
culminating in~\citep{Auer2002}.  Stochastic linear bandits were
formally first introduced in~\citep{Abe2003}, and~\citep{Auer2003UCB}
was the first paper to consider UCB-style algorithms. An algorithm
that is based on a confidence ellipsoid is described
by~\citep{DaniStochasticLinearOptimization2008}, with a variant based
on ridge regression given in~\citep{ChuLRS11}, or explore-then-commit
variant in~\citep{RusmevichientongLinearlyParameterizedBandits2010},
and a variant related to a sparse setting appears
in~\citep{Abbasi-YadkoriPS12}. \citet{AbbasiYadkoriImprovedAlgorithmsLinear2011}
gives an instance dependent bound for linear bandits, which we convert
to the contextual setting.

\paragraph{Differential Privacy.}
Differential privacy, first introduced by
\citet{DworkCalibratingNoiseSensitivity2006,DworkOurData2006}, is a
rigorous mathematical notion of privacy that requires the probability
of any observable output to change very little when any single datum
changes.  (We omit the formal definition, having already defined JDP.)
Among its many elegant traits is the notion of group privacy: should
$k$ datums change then the change in the probability of any event is
still limited by (roughly) $k$ times the change when a single datum
was changed.  Differential privacy also composes: the combination of
$k$ $(\varepsilon,\delta)$-differentially private algorithms is
$\paren[\big]{O(k\varepsilon^2 + 2\sqrt{k\log(\nicefrac 1
    {\delta'})}), k\delta+\delta'}$-differentially private for any
$\delta'>0$~\citep{DworkBoosting2010}.

The notion of differential privacy under continual observation was
first defined by \citet{DworkContinualObservation2010} using the
\textbf{tree-based algorithm}~\citep[originally appearing
in][]{ChanPrivateContinualRelease2010}.  This algorithm maintains a
binary tree whose $n$ leaves correspond to the $n$ entries in the
input sequence. Each node in the tree maintains a noisy
(privacy-preserving) sum of the input entries in its subtree --- the
cumulative sums of the inputs can thus be obtained by combining
at most $\log(n)$ noisy sums.  This algorithm is the key ingredient of
a variety of works that deal with privacy in an online setting,
including counts~\cite{DworkContinualObservation2010}, online convex
optimization~\cite{JainDPOnlineLearning2012}, and regret minimization
in both the
adversarial~\citep{SmithThakurtaPrivateOnlineLearning2013,TossouAchievingPrivacyAdversarial2017}
and
stochastic~\cite{MishraNearlyOptimalDPBandits2015,TossouAlgDPBandits2016}
settings. We comment that \citet{MishraNearlyOptimalDPBandits2015}
proposed an algorithm similar to our own for the contextual bandit
setting, however (i)~without maintaining PSD, (ii)~without any
analysis, only empirical evidence, and (iii)~without presenting lower
bounds. A partial utility analysis of this algorithm, in the reward-privacy model (where the context's privacy is not guaranteed), appears in the recent work of~\citet{NeelR18}. Further details about achieving differential privacy via
additive noise and the tree-based algorithm appear in
\cref{apx_sec:more_background} of the supplementary material.
The related problem of private linear regression has also been extensively studied in the offline setting
\citep{ChaudhuriDPERM2011,BassilyPrivateEmpiricalRisk2014}.

\section{Preliminaries and Notation}%
\label{sec:preliminaries}

We use $\vec{bold}$ letters to denote vectors and bold
$\vec{CAPITALS}$ for matrices.  Given a $d$-column matrix $\vec M$,
its \emph{Gram matrix} is the $(d\times d)$-matrix $\XtX{\vec M}$.  A
symmetric matrix $\vec M$ is positive-semidefinite (PSD, denoted
$\vec M\succeq 0$) if $\transp{\vec x} \vec M \vec x \geq 0$ for any
vector $\vec x$.  Any such $\vec M$ defines a norm on vectors, so we
define $\norm{M}{x}^2 = \transp{\vec x} \vec M \vec x$.  We use
$\vec M\succeq \vec N$ to mean $\vec M-\vec N\succeq 0$.  The Gaussian
distribution $\Normal(\mu,\sigma^2)$ is defined by the density
function
${(2\pi\sigma^2)}^{\nicefrac{-1}{2}}\exp(\nicefrac{-{(x-\mu)}^2}{2\sigma^2})$. The
squared $L_2$-norm of a $d$-dimensional vector whose coordinates are
drawn \iid{} from $\Normal(0,1)$ is given by the $\chi^2(d)$
distribution, which is tightly concentrated around $d$.  Given two
distributions $\Pr$ and $\mathds Q$ we denote their \emph{total
  variation distance} by
$d_{\rm TV}(\mathds P,\mathds Q) = \max_{\text{event
  }E}\abs{\Prob{E}-{\mathds
    Q}(E)}$.  

\paragraph{Notation.}\label{sec:notation} Our bound depends on many
parameters of the problem, specified below. Additional parameters
(bounds) are specified in the assumptions stated below.
\setlength{\multicolsep}{2.0pt plus 1.0pt minus 1pt}
\begin{multicols}{2}[]
  \nolinenumbers{}
  \begin{description}[style=sameline,leftmargin=2em,nosep]
  \item[$n$] horizon, i.e.\ number of rounds
  \item[$s,t$] indices of rounds
  \item[$d$] dimensionality of action space
  \item[$\D_t$] $\subset\Real^d$; decision set at round $t$
  \item[$\vec x_t$] $\in \D_t$; action at round $t$ \todo{Is this necessary? It was specified in the problem definition. Can it be merged with 6?}
  \item[$y_t$] $\in \Real$; reward at round $t$
  \item[$\vec\theta^*$] $\in \Real^d$; unknown parameter vector
  \item[$m$] $\defeq \lceil\log_2(n)+1 \rceil$
  \item[$\vec X_{<t}$] $\in \Real^{(t-1)\times d}$, with
    $\vec X_{<t,s} = \transp{\vec x_s}$ for $s<t$
  \item[$\vec G_t$] Gram matrix of the actions: $\XtX{\vec X_{<t}}$
  \item[$\vec H_t$] regularizer at round $t$
  \item[$\vec y_{<t}$] vector of rewards up to round $t-1$
  \item[$\vec u_t$] action-reward product: $\transp{\vec X_{<t}} \vec y_{<t}$
  \item[$\vec h_t$] perturbation of $\vec u_t$
  \end{description}
  \linenumbers{}
\end{multicols}
\vspace{-2mm}

\section{Linear UCB with Changing Regularizers}%
\label{sec:LinUCB}

In this section we introduce and analyze a variation of
the well-studied LinUCB algorithm, an application of the Upper
Confidence Bound (UCB) idea to stochastic linear bandits
\citep{DaniStochasticLinearOptimization2008,RusmevichientongLinearlyParameterizedBandits2010,AbbasiYadkoriImprovedAlgorithmsLinear2011}.
At every round $t$, LinUCB constructs a \emph{confidence set} $\E_t$
that contains the unknown parameter vector $\vec\theta^*$ with high probability.  It then computes an upper confidence bound on the reward
of each action in the decision set $\D_t$, and ``optimistically''
chooses the action with the highest UCB\@:
$\vec x_t \gets \argmax_{\vec x\in\D_t} \UCB_t(\vec x)$, where
$\UCB_t(\vec x) \defeq \max_{\vec\theta\in\E_t}
\innerp{\vec\theta}{\vec x}$.  We assume the rewards are linear with
added subgaussian noise (i.e.,
$y_s = \innerp{\vec\theta^*}{\vec x_s} + \eta_s$ for $s<t$), so it is
natural to center the confidence set $\E_t$ on the (regularized)
linear regression estimate:
\begin{align*}
  \hat{\vec\theta}_t
  &\defeq \argmin_{\hat{\vec\theta} \in \Real^d} \norm{}{\vec X_{<t} \hat{\vec\theta}
    - \vec y_{<t}}^2 + \norm{\vec H_t}{\hat{\vec\theta}}^2
    = \inv{(\vec G_t + \vec H_t)} \transp{\vec X_{<t}} \vec y_{<t}.
    &\text{where } \vec G_t \defeq \XtX{\vec X_{<t}}
\end{align*}
The matrix $\vec V_t \defeq \vec G_t + \vec H_t \in \Real^{d\times d}$ is a regularized
version of the Gram matrix $\vec G_t$.  Whenever the learner chooses an
action vector $\vec x\in \D_t$, the corresponding reward gives it some information
about the projection of $\vec\theta^*$ onto $\vec x$.  In other
words, the estimate $\hat{\vec\theta}$ is probably closer to
$\vec\theta^*$ along the directions where many actions have
been taken.  This motivates the use of ellipsoidal confidence sets
that are smaller in such directions, inversely
corresponding to the eigenvalues of $\vec G_t$ (or $\vec V_t$). The ellipsoid
is uniformly scaled by $\beta_t$ to achieve the desired confidence
level, as prescribed by \cref{prop:calc-beta}.
\begin{align}\label{eq:def-ellip}
  \E_t &\defeq \set{\vec\theta\in\Real^d \given
        \norm{\vec V_t}{\vec\theta-\hat{\vec\theta}_t} \le \beta_t},
  &\text{for which }
    \UCB_t(\vec x) &= \innerp{\hat{\vec\theta}}{\vec x} + \beta_t\norm{\inv{\vec V_t}}{\vec x}.
\end{align}
Just as the changing regularizer $\vec H_t$ perturbs the Gram matrix $\vec G_t$,
our algorithm allows for the vector
$\vec u_t \defeq \transp{\vec X_{<t}} \vec y_{<t}$ to be perturbed
by $\vec h_t$ to get
$\tilde{\vec u}_t \defeq \vec u_t + \vec h_t$.  The estimate
$\hat{\vec\theta}_t$ is replaced by
$\tilde{\vec\theta}_t \defeq \inv{\vec V_t}\tilde{\vec u}_t$.

\begin{algorithm}[h]
  \caption{Linear UCB with Changing Perturbations}\label{alg:linucb}
  \begin{algorithmic}
    \State{} \textbf{Initialize:} $\vec G_1 \gets \vec 0_{d\times d}$,
    $u_1\gets \vec 0_{d}$.
    \For{each round $t = 1,2,\dotsc,n$}
    \State{} Receive $\D_t \gets{}$ decision set ${} \subset \Real^d$.
    \State{} Receive regularized $\vec V_t \gets \vec G_t + \vec H_t$ and perturbed $\tilde{\vec u}_t \gets \vec u_t + \vec h_t$
    \State{} Compute regressor $\tilde{\vec\theta}_{t} \gets \inv{\vec V_{t}}\tilde{\vec u}_{t}$
    \State{} Compute confidence-set bound $\beta_t$ based on \cref{prop:calc-beta}.
    \State{} Pick action $\vec x_t \gets \argmax_{\vec x\in\D_t}
    \innerp{\tilde{\vec \theta}_t}{\vec x} +
    \beta_t\norm{\inv{\vec V_t}}{\vec x}$.
    \State{} Observe $y_t \gets {}$ reward for action $\vec x_t$
    \State{} Update: $\vec G_{t+1} \gets \vec G_{t} + \vec x_t \transp{\vec x_t},
    \quad \vec u_{t+1} \gets \vec u_{t} + \vec x_t y_t$
    \EndFor{}
  \end{algorithmic}
\end{algorithm}

Our analysis relies on the following assumptions about the environment
and algorithm:
\begin{assumptions*}[\Cref{alg:linucb}]{%
    For all rounds $t=1,\dotsc,n$ and actions $\vec x\in\D_t$:}
  \assume[ass:x-bound] Bounded action set: $\norm{}{\vec x} \le L$.
  \assume[ass:meanreward-bound] Bounded \emph{mean} reward:
  $\abs{\innerp{\vec\theta^*}{\vec x}} \le B$ with $B\ge
  1$.\footnote{%
    See \cref{remark:meanreward-bound} preceding the proof of
    \cref{lemma:linucb-regret} in \cref{sec:more-linucb-discussion} of
    the supplementary material for a discussion as to bounding $B$ by
    $1$ from below.} %
  \assume[ass:thetastar-bound] Bounded target parameter:
  $\norm{}{\vec\theta^*} \le S$.  \assume[ass:psd-reg] All
  regularizers are PSD $\vec H_t \succeq 0$.  \assume[ass:subgaussian]
  $y_t = \innerp{\vec\theta^*}{\vec x_t} + \eta_t$ where $\eta_t$ is
  $\sigma^2$-conditionally subgaussian on previous actions and
  rewards, i.e.:
    $\Ex{\exp(\lambda\eta_t) \given \vec x_1, y_1, \dotsc, \vec x_{t-1}, y_{t-1}, \vec x_t}
    \le \exp(\lambda^2\sigma^2/2)$, for all $\lambda\in\Real$.
  \assume[ass:bounded-data]
  Strongest bound on both the action and the reward:
  $\norm{}{\vec x_t}^2 + y_t^2 \le \tilde L^2$ for all rounds $t$.  In
  particular, if $\norm{}{\vec x_t} \le L$ (\cref{ass:x-bound}) and
  the rewards are bounded: $\abs{y_t} \le \tilde B$ (not just their
  means as in \cref{ass:meanreward-bound}), we can set $\tilde L^2 = L^2 + \tilde B^2$.
\end{assumptions*}

\cref{ass:x-bound,ass:meanreward-bound} aren't required to have good
pseudo-regret bounds, they merely simplify the bounds on the
confidence set (see \cref{prop:calc-beta} and
\cref{sec:regret-bounds}). \cref{ass:bounded-data} is not required at
all for now, it is only used for joint differential privacy in \cref{sec:alg-dp}.


To fully describe \cref{alg:linucb} we need to specify how to compute the confidence-set bounds ${(\beta_t)}_t$. On the one hand, these bounds have to be accurate~--- the confidence set $\E_t$ should contain the unknown $\vec\theta^*$; on the other hand, the larger they are, the larger the regret bounds we obtain.
In other words, the $\beta_t$ should be as small as possible subject to being accurate.

\begin{definition}[Accurate ${(\beta_t)}_t$]\label{def:accurate-beta}
  A sequence ${(\beta_t)}_{t=1}^n$ is $(\alpha, n)$-\emph{accurate} for
  ${(\vec H_t)}_{t=1}^n$ and ${(\vec h_t)}_{t=1}^n$ if, with probability at
  least $1-\alpha$, it satisfies
  $\norm{\vec V_t}{\vec\theta^* - \tilde{\vec\theta}_t} \le \beta_t$
  for all rounds $t=1,\dotsc,n$ simultaneously.
\end{definition}
\vspace{-\parskip}
We now argue that three parameters are the key to establishing accurate confidence-set bounds ${(\beta_t)}_{t=1}^n$~--- taking into account the noise in the setting \emph{and} the noise added by a changing $\vec H_t$ and $\vec h_t$.

\begin{definition}[Accurate $\rho_{\min}$, $\rho_{\max}$, and
  $\gamma$]\label{def:accurate-params}
  The bounds $0 < \rho_{\min} \le \rho_{\max}$ and $\gamma$ are
  $(\alpha/2n)$-\emph{accurate} for ${(\vec H_t)}_{t=1}^n$ and
  ${(\vec h_t)}_{t=1}^n$ if for each round $t$:
  \begin{align*}
    \norm{}{\vec H_t} &\le \rho_{\max},
    &\norm{}{\inv{\vec H_t}} &\le 1/\rho_{\min},
    &\norm{\inv{\vec H_t}}{\vec h_t} &\le \gamma;
    &\text{with probability at least } 1-\alpha/2n.
  \end{align*}
\end{definition}

\begin{restatable}[Calculating $\beta_t$]{proposition}{CalcBeta}%
  \label{prop:calc-beta}
  Suppose
  \cref{ass:thetastar-bound,ass:psd-reg,ass:subgaussian}
  hold and let $\rho_{\min}$, $\rho_{\max}$, and $\gamma$ be
  $(\alpha/2n)$-accurate for some $\alpha\in(0,1)$ and horizon $n$.
  Then ${(\beta_t)}_{t=1}^n$ is $(\alpha,n)$-accurate where
  \begin{align*}
    \beta_t
    &\defeq \sigma\sqrt{2\log(2/\alpha) + \log(\det \vec V_t) - d\log(\rho_{\min})}
      + S\sqrt{\rho_{\max}} + \gamma \\
    &\le \sigma\sqrt{2\log({2}/{\alpha}) + d\log\paren[\Big]{\tfrac{\rho_{\max}}{\rho_{\min}}
      + \tfrac{tL^2}{d\rho_{\min}}}}
      + S\sqrt{\rho_{\max}} + \gamma.
    &\text{(if \cref{ass:x-bound} also holds)}
  \end{align*}
\end{restatable}

\vspace{-\parskip}

\subsection{Regret Bounds}%
\label{sec:regret-bounds}

We now present bounds on the maximum regret of \cref{alg:linucb}.
However, due to space constraints, we defer an extensive discussion of
the proof techniques used and the significance of the results to
\cref{sec:more-linucb-discussion} in the supplementary material.  The
proofs are based on those of
\citet{AbbasiYadkoriImprovedAlgorithmsLinear2011}, who analyzed LinUCB
with constant regularizers.  On one hand, our changes are mostly
technical; however, it turns out that various parts of the proof
diverge and now depend on $\rho_{\max}$ and $\rho_{\min}$; tracing
them all is somewhat involved.  It is an interesting question to
establish similar bounds using known results only as a black box; we
were not able to accomplish this.

\begin{restatable}[Regret of \Cref{alg:linucb}]{theorem}{ThmLinUCBRegret}%
  \label{thm:linucb-regret}
  Suppose \cref{ass:x-bound,ass:meanreward-bound,ass:thetastar-bound,%
    ass:psd-reg,ass:subgaussian} hold and the
  $\beta_t$ are as given by \cref{prop:calc-beta}.  Then with
  probability at least $1-\alpha$ the pseudo-regret of
  \cref{alg:linucb} is bounded by
    \begin{align}
    \label{eq:regret_general_formula}
    \widehat R_n
    &\le B \sqrt{8n}\brck[\bigg]{\sigma\paren*{2\log(\tfrac{2}{\alpha})
      + d\log\paren*{\tfrac{\rho_{\max}}{\rho_{\min}} + \tfrac{nL^2}{d\rho_{\min}}}}
      + (S\sqrt{\rho_{\max}} + \gamma)\sqrt{d\log\paren*{1 + \tfrac{nL^2}{d\rho_{\min}}}}}
  \end{align}
\end{restatable}

\begin{restatable}[Gap-Dependent Regret of
  \Cref{alg:linucb}]{theorem}{ThmLinUCBGapRegret}%
  \label{thm:linucb-gap-regret}
  Suppose \cref{ass:x-bound,ass:meanreward-bound,ass:thetastar-bound,%
    ass:psd-reg,ass:subgaussian} hold and the
  $\beta_t$ are as given by \cref{prop:calc-beta}.  If the optimal
  actions in every decision set $\D_t$ are separated from the
  sub-optimal actions by a reward gap of at least $\Delta$, then with
  probability at least $1-\alpha$ the pseudo-regret of
  \cref{alg:linucb} satisfies
  \begin{align}
  \label{eq:regret_gap_general_formula}
    \widehat R_n
    &\le \frac{8B}{\Delta} \brck[\bigg]{\sigma\paren*{2\log(\tfrac{2}{\alpha})
      + d\log\paren*{\tfrac{\rho_{\max}}{\rho_{\min}} + \tfrac{nL^2}{d\rho_{\min}}}}
      + (S\sqrt{\rho_{\max}}+\gamma)\sqrt{d\log\paren*{1+\tfrac{nL^2}{d\rho_{\min}}}}}^2
  \end{align}
\end{restatable}

\section{Linear UCB with Joint Differential Privacy}%
\label{sec:alg-dp}

Notice that \cref{alg:linucb} uses its history of actions and rewards
up to round $t$ only via the confidence set $\E_t$, which is to say
via $\vec V_t$ and $\tilde{\vec u}_t$, which are perturbations of the Gram
matrix $\vec G_t$ and the vector
$\vec u_t \defeq \transp{\vec X_{<t}} \vec y_{<t}$, respectively;
these also determine $\beta_t$.  By recording this history with
differential privacy, we obtain a Linear UCB algorithm that is jointly
differentially private (\cref{def:joint-dp}) because it simply
post-processes $\vec G_t$ and $\vec u_t$.

\begin{claim}[see {\citet[Proposition~2.1]{DworkAlgorithmicFoundationsDifferential2014}}]
  If the sequence ${(\vec V_t,\tilde{\vec u}_t)}_{t=1}^{n-1}$ is
  $(\varepsilon,\delta)$-differentially private with respect to
  ${(\vec x_t, y_t)}_{t=1}^{n-1}$, then \cref{alg:linucb} is
  $(\varepsilon,\delta)$-jointly differentially private.
\end{claim}

\begin{remark}
  \Cref{alg:linucb} is only jointly differentially private even though
  the history maintains full differential privacy --- its action
  choice depends not only on the past contexts $c_s$ ($s < t$, via the
  differentially private $\vec X_{<t}$) but also on the current
  context $c_t$ via the decision set $\D_t$.  This use of $c_t$
  is \emph{not} differentially private, as it is revealed by the algorithm's
  chosen $\vec x_t$.
\end{remark}

Rather than applying the tree-based algorithm separately to $\vec G_t$
and $\vec u_t$, we aggregate both into the single matrix
$\vec M_t\in\Real^{(d+1)\times(d+1)}$, which we now construct.  Define
$\vec A \defeq \begin{bmatrix} \vec{X}_{1:n} & \vec{y}_{1:n} \end{bmatrix} \in \Real^{n\times(d+1)}$, with $\vec A_t$
holding the top $t-1$ rows of $\vec A$ (and
$\vec A_1 = \vec 0_{1\times(d+1)}$).  Now let
$\vec M_t \defeq \XtX{\vec A_t}$ --- then the top-left $d\times d$
sub-matrix of $\vec M_t$ is the Gram matrix $\vec G_t$ and the first
$d$ entries of its last column are $\vec u_t$.  Furthermore, since
$M_{t+1} = M_t + \XtX{\begin{bmatrix}\transp{\vec x_t} &
    y_t \end{bmatrix}}$, the tree-based algorithm for private
cumulative sums can be used to maintain a private estimation of
$\vec M_t$ using additive noise, releasing $\vec M_t + \vec N_t$.  The
top-left $d\times d$ sub-matrix of $\vec N_t$ becomes $\vec H_t$ and
the first $d$ entries of its last column become $\vec h_t$. Lastly, to
have a private estimation of $\vec M_t$, \Cref{ass:bounded-data} must
hold.

Below we present two techniques for maintaining (and updating) the
private estimations of $\vec M_t$. As mentioned in
\cref{subsec:related_work}, the key component of our technique is
the tree-based algorithm, allowing us to estimate $\vec M_t$ using at most
$m \defeq 1 + \ceil{\log_2n}$ noisy counters. In order for the entire
tree-based algorithm to be $(\varepsilon,\delta)$-differentially
private, we add noise to each node in the tree so that each noisy
count on its own preserves
$(\nicefrac{\varepsilon}{\sqrt{8m\ln(\nicefrac{2}{\delta})}},
\nicefrac{\delta}{2m})$-differential privacy. Thus in each day, the
noise $\vec N_t$ that we add to $\vec M_t$ comes from the sum of at most $m$
such counters.



\subsection{Differential Privacy via Wishart Noise}%
\label{sec:dp-wishart}

First, we instantiate the tree-based algorithm with noise from a
suitably chosen Wishart distribution $\Wishart_{d+1}(\vec V, k)$, which is the
result of sampling $k$ independent $(d+1)$-dimensional Gaussians from
$\Normal(\vec 0_{d+1}, \vec V)$ and computing their Gram
matrix.

\begin{theorem}[{Theorem~4.1~\citep{SheffetPrivateApproxRegression2015}}]%
  \label{thm:wishart-cont-dp}
  Fix positive $\varepsilon_0$ and $\delta_0$. If the $L_2$-norm of each
  row in the input is bounded by $\tilde L$ then releasing the input's
  Gram matrix with added noise sampled from
  $\Wishart_{d+1}(\tilde L^2 \Eye, k_0)$ is
  $(\varepsilon_0,\delta_0)$-differentially private, provided
  $k_0 \ge d + 1 + 28\varepsilon_0^{-2}\ln(\nicefrac{4}{\delta_0})$.
\end{theorem}
\vspace{-\parskip} Applying this guarantee to our setting, where each
count needs to preserve
$(\nicefrac{\varepsilon}{\sqrt{8m\ln(\nicefrac{2}{\delta})}},
\nicefrac{\delta}{2m})$-differential privacy, it suffices to sample a
matrix from $\Wishart_{d+1}(\tilde{L} \Eye, k)$ with
$k \defeq d + 1 + \ceil{224
  m\varepsilon^{-2}\ln(\nicefrac{8m}{\delta})\ln(\nicefrac{2}{\delta})}$.
Moreover, the sum of $m$ independent samples from the Wishart
distribution is a noise matrix
$\vec N_t\sim \Wishart_{d+1}(\tilde L^2\Eye, mk)$.%
\footnote{Intuitively, we merely concatenate the $m$ batches of
  multivariate Gaussians sampled in the generation of each of the $m$
  Wishart noises.} %
Furthermore, consider the regularizers $\vec H_t$ and $\vec h_t$ derived
from $\vec N_t$ (the top-left submatrix and the right-most subcolumn
resp.)~--- $\vec H_t$ has distribution
$\Wishart_d(\tilde L^2\Eye,mk)$, and each entry of $\vec h_t$ is the
dot-product of two multivariate Gaussians. Knowing their distribution,
we can infer the accurate bounds required for our regret bounds. 
Furthermore, since the Wishart noise has eigenvalues that are fairly large, we consider a post-processing of the noise matrix~--- shifting it by $-c\Eye$ with
\begin{equation}
\label{eq:c}
c \defeq \tilde L^2 \paren[\big]{\sqrt{mk} - \sqrt{d} - \sqrt{2\ln(8n/\alpha)}}^2-4\tilde L^2\sqrt{mk}\paren[\big]{\sqrt d + \sqrt{2\ln(8n/\alpha)}}
\end{equation}
making the bounds we require smaller than without the shift.
The derivations are deferred to \cref{apx_sec:privacy_proofs} of the
supplementary material.
\begin{restatable}{proposition}{PropWishartTails}%
\label{pro:accurate_bounds_for_Wishart}
Fix any $\alpha>0$. If for each $t$ the $\vec H_t$ and $\vec h_t$ are
generated by the tree-based algorithm with Wishart noise
$\Wishart_{d+1}(\tilde L^2 \Eye, k)$, then the following
are $(\alpha/2n)$-accurate bounds:
\begin{align*}
  \rho_{\min} &= \tilde L^2 \paren[\big]{\sqrt{mk}
  - \sqrt{d} - \sqrt{2\ln(8n/\alpha)}}^2,\\
  \rho_{\max} &= \tilde L^2 \paren[\big]{\sqrt{mk}
  + \sqrt{d} + \sqrt{2\ln(8n/\alpha)}}^2,\\
  \gamma &= \tilde L \paren[\big]{\sqrt{d} + \sqrt{2\ln(2n/\alpha)}}.
\end{align*}
Moreover, if we use the shifted regularizer $\vec H_t' \defeq \vec H_t -
c\Eye$ with $c$ as given in \cref{eq:c}, then the following
are $(\alpha/2n)$-accurate bounds:
\begin{align*}
  \rho_{\min}' &= 4\tilde L^2\sqrt{mk}\paren[\big]{\sqrt d + \sqrt{2\ln(8n/\alpha)}},\\
  \rho_{\max}' &= 8\tilde L^2\sqrt{mk}\paren[\big]{\sqrt d + \sqrt{2\ln(8n/\alpha)}},\\
  \gamma' &= \tilde L \sqrt{\sqrt{mk} \paren[\big]{\sqrt{d} + \sqrt{2\ln(2n/\alpha)}}}.
 \end{align*}
\end{restatable}

Plugging these into \cref{thm:linucb-regret,thm:linucb-gap-regret}
gives us the following upper bounds on pseudo-regret.
\begin{corollary}%
\label{cor:regret_with_Wishart}
\cref{alg:linucb} with $\vec H_t$ and $\vec h_t$ generated by the
tree-based mechanism with each node adding noise independently from
$\Wishart_{d+1}( (L^2+\tilde B^2)\Eye, k )$ and then subtracting
$c\Eye$ using \cref{eq:c}, we get a pseudo-regret bound of
\begin{align*}
O\paren*{B \sqrt{n}\brck[\bigg]{\sigma\paren*{\log\paren[\big]{\nicefrac{1}{\alpha}}
      + d\log\paren*{n}}
      + S\tilde L \sqrt d {\log(n)}^{3/4}(d^{1/4} + \varepsilon^{-1/2}{\log(\nicefrac 1 \delta)}^{1/4})(d^{1/4}+{\log(\nicefrac n \alpha)}^{1/4})}}
\end{align*}

in general, and a gap-dependent pseudo-regret bound of
\begin{align*}
O\paren*{\tfrac B \Delta\brck[\bigg]{\sigma\paren*{\log\paren[\big]{\nicefrac{1}{\alpha}}
      + d\log\paren*{n}}
      + S\tilde L \sqrt d {\log(n)}^{3/4}(d^{1/4} + \varepsilon^{-1/2}{\log(\nicefrac 1 \delta)}^{1/4})(d^{1/4}+{\log(\nicefrac n \alpha)}^{1/4})}^2}
\end{align*}

\end{corollary}

\subsection{Differential Privacy via Additive Gaussian Noise}%
\label{sec:dp-gauss}

Our second alternative is to instantiate the tree-based algorithm with
symmetric Gaussian noise: sample $\vec Z'\in\Real^{(d+1)\times(d+1)}$
with each $\vec Z_{i,j}' \sim \Normal(0,\sigma_{\rm noise}^2)$ \iid{}
and symmetrize to get $\vec Z = (\vec Z'+\transp{\vec Z'})/\sqrt 2$.%
\footnote{This increases the variance along the diagonal entries
  beyond the noise magnitude required to preserve privacy, but only by
  a constant factor of $2$.} %
Recall that each datum has a bounded $L_2$-norm of $\tilde L$, hence a
change to a single datum may alter the Frobenius norm of $\vec M_t$ by
$\tilde L^2$. It follows that in order to make sure each node in the
tree-based algorithm preserves
$(\nicefrac{\varepsilon}{\sqrt{8m\ln(\nicefrac{2}{\delta})}},
\nicefrac{\delta}{2})$-differential privacy,%
\footnote{We use here the slightly better bounds for the composition
  of Gaussian noise based on zero-Concentrated
  DP~\citep{BunConcentratedDifferentialPrivacy2016}.} %
the variance in each coordinate must be
$\sigma_{\rm noise}^2 = 16m\tilde L ^4 {\ln(\nicefrac 4 \delta)}^2 /
\varepsilon^2$.  When all entries on $\vec Z$ are sampled from
$\Normal(0,1)$, known concentration
results~\cite{TaoRandomMatrixTheory2012} on the top singular value of
$\vec Z$ give that
$\Pr[ \|\vec Z\| > (4\sqrt{d+1} + 2\ln(\nicefrac {2n} \alpha)) ] <
\nicefrac{\alpha}{2n}$. Note however that in each day $t$ the noise
$\vec N_t$ is the sum of $\leq m$ such matrices, thus the variance of
each coordinate is $m\sigma_{\rm noise}^2$.  The top-left
$(d\times d)$-submatrix of $\vec N_t$ has operator norm of at most
\begin{align*}
  \Upsilon &\defeq \sigma_{\rm noise}\sqrt{2m}\paren[\big]{4\sqrt{d} + 2\ln(2n/\alpha)}
  = \sqrt{32} m\tilde L ^2 \ln(4/\delta) \paren[\big]{4\sqrt{d} + 2\ln(2n/\alpha)} / \varepsilon.
\end{align*}
However, it is important to note that the result of adding Gaussian
noise may not preserve the PSD property of the noisy Gram matrix. To
that end, we ought to shift $\vec N_t$ by some $c\Eye$ in order
to make sure that we maintain strictly positive eigenvalues throughout
the execution of the algorithm.  Since the bounds in \cref{thm:linucb-regret,thm:linucb-gap-regret} mainly depend on $\sqrt{\rho_{\max}} +\gamma$, we choose the shift-magnitude to be $2\Upsilon \Eye$. This makes $\rho_{\max}=3\Upsilon$ and $\rho_{\min}=\Upsilon$ and as a result
$\norm{\inv{\vec H_t}}{\vec h_t} \leq \sqrt{\inv \Upsilon}\norm{}{\vec h_t}$, which we can
bound using standard concentration bounds on the $\chi^2$-distribution
(see \cref{claim:chi2-tails}).  This culminates in the following bounds.
\begin{proposition}%
\label{pro:accurate_bounds_for_Gaussian}
Fix any $\alpha>0$. Given that for each $t$ the regularizers
$\vec H_t, \vec h_t$ are taken by applying the tree-based algorithm
with symmetrized shifted Gaussian noise whose entries are sampled
\iid{} from $\Normal(0, \sigma_{\rm noise}^2)$, then the following
$\rho_{\min}$, $\rho_{\max}$, and $\gamma$ are $(\alpha/2n)$-accurate
bounds:
\begin{align*}
  \rho_{\min} = \Upsilon, \quad
  \rho_{\max} = 3\Upsilon, \quad
  \gamma = \sigma_{\rm noise}\sqrt{\inv{\Upsilon}m} \paren[\big]{\sqrt d + \sqrt{2\ln(\nicefrac{2n}\alpha)}}
  \leq \sqrt{{m\tilde L^2 \paren[\big]{\sqrt d+2\ln(2n/\alpha)}} / \paren[\big]{\sqrt 2 \varepsilon}}
  \end{align*}
\end{proposition}
Note how this choice of shift indeed makes both $\rho_{\max}$ and
$\gamma^2$ roughly on the order of $O(\Upsilon)$.

The end result is that for each day $t$, $\vec h_t$ is given by summing
at most $m$ $d$-dimensional vectors whose entries are sampled \iid{}
from $\Normal(0,\sigma_{\rm noise}^2)$; the symmetrization doesn't
change the distribution of each coordinate.  The matrix $\vec H_t$ is
given by
\begin{enumerate*}[(i),before=\unskip{: },itemjoin={{; }},itemjoin*={{; and }}]
\item summing at most $m$ matrices whose entries are sampled \iid{} from
  $\Normal(0,\sigma_{\rm noise}^2)$
\item symmetrizing the result as
  shown above
\item adding $2\Upsilon\Eye$.
\end{enumerate*}
This leads to a bound on the regret of \cref{alg:linucb} with the
tree-based algorithm using Gaussian noise.
\begin{corollary}%
\label{cor:regret_with_Gaussian}
Applying \cref{alg:linucb} where the regularizers $\vec H_t$ and $\vec h_t$ are derived by applying the tree-based algorithm where each node holds a symmetrized matrix whose entries are sampled \iid{} from $\Normal(0,\sigma_{\rm noise}^2)$ and adding $2\Upsilon\Eye$, we get a regret bound of
\begin{align*}
O\paren*{B\sqrt n \paren*{
\sigma (d \log(n)+\log(\nicefrac 1 \alpha)) + S\tilde L \log(n) \sqrt{
{d(\sqrt d + \ln(\nicefrac {n}{\alpha}))\ln(\nicefrac 1 \delta)}/{\varepsilon}
}}}
\end{align*}
in general, and a gap-dependent pseudo-regret bound of
\begin{align*}
O\paren*{\tfrac B \Delta \paren*{  
\sigma (d \log(n)+\log(\nicefrac 1 \alpha)) + S\tilde L \log(n) \sqrt{
{d(\sqrt d + \ln(\nicefrac {n}{\alpha}))\ln(\nicefrac 1 \delta)}/{\varepsilon}
}}^2}
\end{align*}

\end{corollary}
\vspace{-\parskip}

\section{Lower Bounds}%
\label{sec:lower-bounds}

In this section, we present lower bounds for two versions of the
problem we deal with in this work. The first, and probably the more
obvious of the two, deals with the impossibility of obtaining
sub-linear regret for the contextual bandit problem with the standard
notion of differential privacy (under continual observation). Here, we
assume user $t$ provides a context $c_t$ which actually determines the
mapping of the arms into feature vectors: $\phi(c_t, a) \in
\Real^d$. The sequence of choice thus made by the learner is
$a_1, \dotsc, a_n \in \A^n$ which we aim to keep private. The next
claim, whose proof is deferred to \cref{apx_sec:privacy_proofs}
in the supplementary material, shows that this is impossible without
effectively losing any reasonable notion of utility.
\DeclareRobustCommand{\DPDefinition} {Formally, two sequences
  $S = \langle (c_1, y_1), \dotsc, (c_n,y_n)\rangle$ and
  $S' = \langle (c'_1, y'_1), \dotsc, (c'_n,y'_n)\rangle$ are called
  neighbors if there exists a single $t$ such that for any $t'\neq t$
  we have $(c_{t'},y_{t'}) = (c'_{t'},y'_{t'})$; and an algorithm $A$
  is $(\varepsilon,\delta)$-differentially private if for any two
  neighboring sequences $S$ and $S'$ and any subsets of sequences of
  actions ${\cal S}\subset \A^n$ it holds that
  $\Pr[A(S)\in\mathcal{S}] \leq e^\varepsilon \Pr[A(S')\in
  \mathcal{S}] +\delta$.}

\begin{restatable}{claim}{clmLinearRegretDP}%
\label{clm:standard_contextual_DP_implies_linear_regret}
For any $\varepsilon < \ln(2)$ and $\delta < 0.25$, any $(\varepsilon,\delta)$-differentially private algorithm $A$ for the contextual bandit problem must incur pseudo-regret of $\Omega(n)$.
\end{restatable}
\vspace{-\parskip}
\DeclareRobustCommand{\DPLowerBoundProof}
{\begin{proof}
We consider a setting with two arms $\A=\{a^1,a^2\}$ and two possible contexts: $c^1$ which maps $a^1\mapsto \vec\theta^*$ and $a^2 \mapsto -\vec\theta^*$; and $c^2$ which flips the mapping. Assuming $\|\vec\theta^*\|=1$ it is evident we incur a pseudo-regret of $2$ when pulling arm $a^1$ is under context $c^2$ or pulling arm $a^2$ under $c^1$. Fix a day $t$ and a history of previous inputs and arm pulls $H_{t-1}$. Consider a pair of neighboring sequences that agree on the history $H_{t-1}$ and differ just on day $t$ --- in $S$ the context $c_t =c^1$ whereas in $S'$ it is set as $c_t =c^2$. Denote $\mathcal{S}$ as the subset of action sequences that are fixed on the first $t-1$ days according to $H_{t-1}$, have the $t$-th action be $a^1$ and on days $>t$ may have any action. Thus, applying the guarantee of differential privacy w.r.t to $\mathcal{S}$ we get that $\Pr[a_t = a^1 |~S] = \Pr[A(S)\in \mathcal{S}] \leq e^\varepsilon \Pr[a_t=a^1 |~S'] + \delta$. Consider an adversary that sets the context of day $t$ to be either $c^1$ or $c^2$ uniformly at random and independently of other days. The pseudo-regret incurred on day $t$ is thus: $2\cdot \tfrac 1 2 \left( \Pr[a_t=a^2|~S] + \Pr[a_t=a^1|~S'] \right)  \geq (1- \Pr[a_t=a^1|~S]) + e^{-\varepsilon}(\Pr[a_t=a^1|~S]-\delta) = 1 + (e^{-\varepsilon}-1)\Pr[a_t=a^2|~S] - \delta > 1 - 1\cdot \tfrac 1 2 - \tfrac 1 4 = \tfrac 1 4$. As the above applies to any day $t$,  the algorithm's pseudo-regret is $\geq \tfrac n 4$ against such random adversary.
\end{proof}} The second lower bound we show is more challenging. We
show that any $\varepsilon$-differentially private algorithm for the
classic MAB problem must incur \emph{an additional} pseudo-regret of
$\Omega( k\log(n)/\epsilon)$ on top of the standard (non-private)
regret bounds.  We consider an instance of the MAB where the leading
arm is $a^1$, the rewards are drawn from a distribution over
$\{-1,1\}$, and the gap between the means of arm $a^1$ and arm
$a\neq a^1$ is $\Delta_a$. Simple calculation shows that for such
distributions, the total-variation distance between two distributions
whose means are $\mu$ and $\mu-\Delta$ is $\Delta/2$. Fix
$\Delta_2, \Delta_3, \dotsc, \Delta_k$ as some small constants, and we
now argue the following.
\begin{claim}%
\label{clm:LB_pulling_any_arms_many_times}
Let $A$ be any $\varepsilon$-differentially private algorithm for the MAB problems with $k$ arms whose expected regret is at most $n^{3/4}$. Fix any arm $a\neq a^1$, whose difference between it and the optimal arm $a^1$ is $\Delta_a$. Then, for sufficiently large $n$s, $A$ pulls arm $a$ at least $\nicefrac {\log(n)}{100\varepsilon\Delta_a}$ many times w.p. $\geq \nicefrac 1 2$.
\end{claim}
\vspace{-\parskip}
We comment that the bound $n^{3/4}$ was chosen arbitrarily, and we only require a regret upper bound of $n^{1-c}$ for some $c>0$. Of course, we could have used standard assumptions, where the regret is asymptotically smaller than \emph{any} polynomial; or discuss algorithms of regret $\tilde O(\sqrt n)$ (best minimax regret). Aiming to separate the standard lower-bounds on regret from the private bounds, we decided to use $n^{3/4}$.
As an immediate corollary we obtain the following \emph{private} regret bound:
\begin{corollary}%
\label{cor:LB_private_MAB}
The expected pseudo-regret of any $\varepsilon$-differentially private algorithm for the MAB is $\Omega(k\log(n)/\varepsilon)$. Combined with the non-private bound of $\Omega\paren[\big]{\sum_{a\neq a^1} \nicefrac{\log(n)}{\Delta_a}}$ we get that the private regret bound is the $\max$ of the two terms, i.e.: $\Omega\paren[\big]{\nicefrac {k\log(n)}{\varepsilon}+\sum_{a\neq a^1} \nicefrac{\log(n)}{\Delta_a}}$.
\end{corollary}
\vspace{-2.5\parskip}
\begin{proof}
  Based on \cref{clm:LB_pulling_any_arms_many_times}, the expected
  pseudo-regret is at least
  $\sum\limits_{a\neq a^1}\tfrac{\Delta_a\log(n)}{200
    \varepsilon\Delta_a} = \tfrac{(k-1)\log(n)}{200\varepsilon}$.
\end{proof}
\vspace{-2.5\parskip}
\begin{proof}[Proof of \Cref{clm:LB_pulling_any_arms_many_times}]
Fix arm $a$. Let $\bar P$ be the vector of the $k$-probability distributions associated with the $k$ arms. Denote $E$ as the event that arm $a$ is pulled $<\nicefrac{\log(n)}{100\varepsilon\Delta_a} := t_a$ many times. Our goal is to show that $\Pr_{A;~{\rm rewards}\sim \bar P}[E]<\nicefrac 1 2$.

To that end, we postulate a different distribution for the rewards of
arm $a$ --- a new distribution whose mean is \emph{greater} by
$\Delta_a$ than the mean reward of arm $a^1$. The total-variation
distance between the given distribution and the postulated
distribution is $\Delta_a$. Denote $\bar Q$ as the vector of
distributions of arm-rewards (where only $P_a \neq Q_a$). We now argue
that should the rewards be drawn from $\bar Q$, then the event $E$ is
very unlikely: $\Pr_{A;~{\rm rewards}\sim \bar Q}[E] \leq 2n^{-1/4}/\Delta_a$. Indeed, the argument is based on a standard Markov-like argument: the expected pseudo-regret of $A$ is at most $n^{3/4}$, yet it is at least $\Pr_{A;~{\rm rewards}\sim \bar Q}[E]\cdot (n-t_a) \Delta_a \geq  (n\Delta_a/2)\Pr_{A;~{\rm rewards}\sim \bar Q}[E]$, for sufficiently large $n$.

We now apply a beautiful result of~\citet[Lemma~6.1]{KarwaVadhanFiniteSampleDP2017}, stating that the ``effective'' group privacy between the case where the $n$ datums of the inputs are drawn \iid{} from either distribution $P$ or from distribution $Q$ is proportional to $\varepsilon n\cdot d_{\rm TV}(P,Q)$. In our case, the key point is that we only consider this change \emph{under the event $E$}, thus the number of samples we need to redraw from the distribution $P_a$ rather than $Q_a$ is strictly smaller than $t_a$, and the elegant coupling argument of~\cite{KarwaVadhanFiniteSampleDP2017} reduces it to $6\Delta_a \cdot t_a$. To better illustrate the argument, consider the coupling argument of~\cite{KarwaVadhanFiniteSampleDP2017} as an oracle $\mathcal{O}$. The oracle generates a collection of precisely $t_a$ \emph{pairs} of points, the left ones are \iid{} samples from $P_a$ and the right ones are \iid{} samples from $Q_a$, and, in expectation, in $(1-\Delta_a)$ fraction of the pairs the right- and the left-samples are identical. Whenever the learner $A$ pulls arm $a$ it makes an oracle call to $\mathcal{O}$, and depending on the environment (whether the distribution of rewards is $\bar P$ or $\bar Q$) $\mathcal{O}$ provides either a fresh left-sample or a right-sample. Moreover, suppose there exists a counter $\mathcal{C}$ that stands between $A$ and $\mathcal{O}$, and in case $\mathcal{O}$ runs out of examples then $\mathcal{C}$ routes $A$'s oracle calls to a different oracle. Now,~\citet[Lemma~6.1]{KarwaVadhanFiniteSampleDP2017} assures that the probability of the event ``$\mathcal{C}$ never re-routes the requests'' happens with similar probability under $P$ or under $Q$, different only up to a multiplicative factor of $\exp(\epsilon \Delta_a t_a)$. And seeing as the event ``$\mathcal{C}$ never re-routes the requests'' is quite unlikely when $\mathcal{O}$ only provides right-samples (from $\bar Q$), it is also fairly unlikely when $\mathcal{O}$ only provides left-samples (from $\bar P$).

Formally, we conclude the proof by applying the result
of~\cite{KarwaVadhanFiniteSampleDP2017} to infer that $\Pr_{A;~{\rm
    rewards}\sim \bar P}[E] \leq \exp(6\varepsilon\Delta_a
t_a)\Pr_{A;~{\rm rewards}\sim \bar Q}[E] \leq \exp(0.06 \log(n)) \cdot
\tfrac 2{\Delta_a}n^{-1/4} = n^{-0.19}\tfrac 2{\Delta_a} \leq \nicefrac 1
2$ for sufficiently large $n$s, proving the required claim.
\end{proof}




\subsubsection*{Acknowledgements}

We gratefully acknowledge the Natural Sciences and Engineering
Research Council of Canada (NSERC) for supporting R.S. with the
Alexander Graham Bell Canada Graduate Scholarship and O.S. with grant
\#2017--06701.  R.S. was also supported by Alberta Innovates and
O.S. is also an unpaid collaborator on NSF grant \#1565387.

\bibliographystyle{plainnat}
{\small\bibliography{references,zotero-references}}

\cleardoublepage{}
\appendix
\phantomsection{}
\addcontentsline{toc}{chapter}{Supplementary Material}
\begin{center}
  \LARGE\bf Supplementary Material
\end{center}

\section{Additional Background Information}%
\label[section]{apx_sec:more_background}

\subsection{Differential Privacy}%
\label[subsection]{apx_subsec:DP}

In the offline setting, a dataset $D$ is a $n$-tuple of elements from
some universe $\mathcal{U}$. Two datasets are called neighbors if they
differ just on a single element. An algorithm $A$ is said to be
$(\varepsilon,\delta)$-differentially private if for any pair of
neighboring datasets $D$ and $D'$ and any subset of possible outputs
$\mathcal{S}$ we have that
$\Pr[A(D)\in\mathcal{S}]\leq e^\varepsilon\Pr[A(D')\in\mathcal{S}] +
\delta$. A common technique~\cite{DworkOurData2006} for approximating
the value of a query $f$ on dataset $D$ is to first find its
$L_2$-sensitivity,
$GS_2 :=\max_{D,D'\text{\ neighboring}} \|f(D)-f(D')\|_2$, and then add
Gaussian noise of $0$-mean and variance
$\tfrac{2GS_2^2\ln(\nicefrac 2 \delta)}{\varepsilon^2}$.

\subsection{The Tree-Based Mechanism}

Assume for simplicity that $n=2^i$ for some positive integer $i$. Let
$T$ be a complete binary tree with its leaf nodes being $l_1, \dotsc,l_n$. Each internal
node $x\in T$ stores the sum of all the leaf nodes in the tree
rooted at $x$. First notice that one can compute any partial sum $\sum_{j=1}^i l_i$ using at most $m:=\lceil\log(n)+1\rceil$ nodes of $T$. Second, notice that for any two neighbor-
ing data sequences $D$ and $D'$ the partial sums stored at no more than $m$ nodes in $T$ are different. Thus, if the count in each node preserves $(\varepsilon_0,\delta_0)$-differential privacy, using the advanced composition of~\cite{DworkBoosting2010} we get that the entire algorithm is $\left(O(m\varepsilon_0^2+\varepsilon_0\sqrt{2m\ln(\nicefrac 1{\delta'})}), m\delta_0 + \delta'\right)$-differentially private. Alternatively, to make sure the entire tree is $(\varepsilon,\delta)$-differentially private, it suffices to set $\varepsilon_0 = \varepsilon / \sqrt{8m\ln(\nicefrac 2 \delta)}$ and $\delta_0= \tfrac \delta{2m}$ (with $\delta'=\nicefrac \delta 2$).

\subsection{Useful Facts.}%
\label[subsection]{apx_subsec:facts}

In this work, we repeatedly apply the following facts about PSD matrices, the Gaussian distribution, the $\chi^2$-distribution and the Wishart-distribution.

\begin{claim}[{\citealp[Theorem~7.8]{ZhangMatrixTheory2011}}]%
  \label{claim:psd-matrix-props}%
  If $\vec A \succeq \vec B \succeq 0$, then
  \begin{enumerate}[nolistsep]
  \item $\rank(\vec A) \ge \rank(\vec B)$
  \item $\det \vec A \ge \det \vec B$
  \item $\inv{\vec B} \succeq \inv{\vec A}$ if $\vec A$ and $\vec B$ are nonsingular.
  \end{enumerate}
\end{claim}

\begin{claim}[{\citealp[Corollary to
    Lemma~1,][p.~1325]{LaurentAdaptiveEstimation2000}}]%
  \label{claim:chi2-tails}
  If $U\sim\chi^2(d)$ and $\alpha\in(0,1)$,
  \begin{align*}
    \Prob[\bigg]{U \ge d + 2\sqrt{d\ln\tfrac{1}{\alpha}} + 2\ln\tfrac{1}{\alpha}} &\le \alpha,
    & \Prob[\bigg]{U \le d - 2\sqrt{d\ln\tfrac{1}{\alpha}}} &\le \alpha.
  \end{align*}
\end{claim}

\begin{claim}[{\citealp[Adaptation
    of][Corollary~5.35]{VershyninRandomMatrices2010}}]%
  \label{claim:gaussian-matrix-tails}
  Let $\vec A$ be an $n\times d$ matrix whose entries are independent
  standard normal variables.  Then for every $\alpha\in(0,1)$, with
  probability at least $1-\alpha$ it holds that
  \begin{align*}
    \sigma_{\min}(\vec A), \sigma_{\max}(\vec A) &\in \sqrt{n} \pm (\sqrt{d} + \sqrt{2\ln(2/\alpha)})
  \end{align*}
  with $\sigma_{\min}(\vec A)$ and $\sigma_{\max}(\vec A)$ denoting the smallest- and largest singular values of $\vec A$ resp.
\end{claim}

\begin{claim}[{\citealp[Lemma~A.3]{SheffetPrivateApproxRegression2015}}]%
  \label{claim:wishart-tails}
  Fix $\alpha\in(0,1/e)$ and let $\vec W\sim\Wishart_d(\vec V, k)$ with $\sqrt{m}
  > \sqrt{d} + \sqrt{2\ln(2/\alpha)}$.  Then, denoting the $j$-t largest eigenvalue of $\vec W$ as $\sigma_j(\vec W)$, with probability at least
  $1-\alpha$ it holds that for every $j = 1,2,\dotsc,d$:
  \begin{align*}
    \sigma_j(\vec W) &\in \paren*{\sqrt{m} \pm \paren[\Big]{\sqrt{d} + \sqrt{2\ln(2/\alpha)}}}^2 \sigma_j(\vec V).
  \end{align*}
\end{claim}

\begin{claim}\label{claim:shifted-matrix-norm}
  For any matrix $\vec H \succeq \rho\Eye \succ \vec 0$, vector
  $\vec v$, and constant $c \ge 0$ satisfying $\vec H - c\Eye \succ 0$,
  \begin{align*}
    \norm{\inv{(\vec H-c\Eye)}}{\vec v}
    &\le \norm{\inv{\vec H}}{\vec v} \sqrt{\frac{\rho}{\rho-c}}.
  \end{align*}

  \begin{proof}
    Since $\vec 0 \prec \rho\Eye \preceq \vec H$, an application of
    \cref{claim:psd-matrix-props} gives
    \begin{align*}
      c\Eye &\preceq (c/\rho)\vec H &\text{multiplying by } c/\rho \ge 0 \\
      \vec H - c\Eye &\succeq \vec H - (c/\rho) \vec H = \paren*{\frac{\rho-c}{\rho}} \vec H \\
      \inv{(\vec H-c\Eye)} &\preceq \paren*{\frac{\rho}{\rho-c}}\inv{\vec H}.
    \end{align*}
    The result follows from the definition of $\norm{\inv{\vec H}}{\vec v}$.
  \end{proof}
\end{claim}

\section{Discussion and Proofs from Section~\ref{sec:LinUCB}}%
\label[section]{sec:more-linucb-discussion}

The proofs in this section are based on those of
\citet{AbbasiYadkoriImprovedAlgorithmsLinear2011}, who analyze the
LinUCB algorithm with a constant regularizer.  The main difference
from that work is that, in our case, quantities involving the
regularizer must be bounded above or below (as appropriate) by the
constants $\rho_{\max}$ and $\rho_{\min}$, respectively.  We will make
extensive use of \cref{claim:psd-matrix-props}, which shows that for
any two matrices $\vec 0 \preceq \vec V \preceq \vec U$, we have
$\det\vec V \le \det\vec U$ and $\inv{\vec V} \succeq \inv{\vec U}$.
We start by proving the following proposition about the sizes of the
confidence ellipsoids, which illustrates this general idea.

\CalcBeta*

\begin{proof}
  By definition, $\tilde{\vec\theta}_t = \inv{\vec V_t} \tilde{\vec u}_t$,
  $\tilde{\vec u}_t = \vec u_t + \vec h_t$, and $\vec u_t = \transp{\vec{X}_{<t}} \vec y_{<t}$, so that
  \begin{align*}
    \vec\theta^* - \tilde{\vec\theta}_t
    &= \vec\theta^* - \inv{\vec V_t}(\transp{\vec X_{<t}} \vec y_{<t} + \vec h_t) \\
    &= \vec\theta^* - \inv{\vec V_t}(\XtX{\vec X_{<t}} \vec\theta^*
      + \transp{\vec X_{<t}} \vec\eta_{t-1} + \vec h_t)
    &\text{since } \vec y_{<t} = \vec X_{<t} \vec\theta^* + \vec\eta_{t-1} \\
    &= \vec\theta^* - \inv{\vec V_t}(\vec V_t \vec\theta^* - \vec H_t \vec\theta^* + \vec z_t + \vec h_t)
    &\text{defining } \vec z_t \defeq \transp{\vec X_{<t}} \vec\eta_{t-1} \\
    &= \inv{\vec V_t}(\vec H_t\vec\theta^* - \vec z_t - \vec h_t) \\
    \shortintertext{Multiplying both sides by $\vec V_t^{1/2}$ gives}
    \vec V_t^{1/2}(\vec\theta^* - \tilde{\vec\theta}_t)
    &= \vec V_t^{-1/2}(\vec H_t\vec\theta^* - \vec z_t - \vec h_t) \\
    \norm{\vec V_t}{\vec\theta^* - \tilde{\vec\theta}_t}
    &= \norm{\inv{\vec V_t}}{\vec H_t\vec\theta^* - \vec z_t - \vec h_t}
    &\text{applying $\norm{}{\wildcard}$ to both sides}\\
    &\le \norm{\inv{\vec V_t}}{\vec H_t\vec\theta^*} + \norm{\inv{\vec V_t}}{z_t}
      + \norm{\inv{\vec V_t}}{\vec h_t} &\text{triangle inequality} \\
    &\le \norm{\inv{\vec V_t}}{\vec z_t} + \norm{\inv{\vec H_t}}{\vec H_t\vec\theta^*}
      + \norm{\inv{\vec H_t}}{\vec h_t}
    &\text{by \cref{claim:psd-matrix-props} since } \vec V_t \succeq \vec H_t \\
    &= \norm{\inv{(\vec G_t+\rho_{\min}\Eye)}}{\vec z_t} + \norm{\vec H_t}{\vec\theta^*} +
      \norm{\inv{\vec H_t}}{\vec h_t},
    &\text{since } \vec V_t \succeq \vec G_t + \rho_{\min}\Eye.
  \end{align*}
  We use a union bound over all $n$ rounds to bound
  $\norm{\vec H_t}{\vec\theta^*} \le
  \sqrt{\norm{}{\vec H_t}}\norm{}{\vec\theta^*} \le S\sqrt{\rho_{\max}}$
  and $\norm{\inv{\vec H_t}}{\vec h_t} \le \gamma$ with probability at
  least $1-\alpha/2$.  Finally, by the ``self-normalized bound for
  vector-valued martingales'' of
  \citet[Theorem~1]{AbbasiYadkoriImprovedAlgorithmsLinear2011}, with
  probability $1-\alpha/2$ for all rounds simultaneously
  \begin{align*}
    \norm{\inv{(\vec G_t + \rho_{\min}\Eye)}}{\vec z_t}
    &\le \sigma\sqrt{2\log\frac{2}{\alpha} + \log\frac{\det(\vec G_t+\rho_{\min}\Eye)}{\det\rho_{\min}\Eye}}
      \le \sigma\sqrt{2\log\frac{2}{\alpha} + \log\det \vec V_t - d\log\rho_{\min}}.
  \end{align*}
  It only remains to show the upper-bound on each $\beta_t$.  By
  \cref{claim:psd-matrix-props}, we have
  $\det \vec V_t = \det(\vec G_t+\vec H_t) \le \det(\vec G_t+\rho_{\max}\Eye)$ and
  \begin{align*}
    \log\det \vec V_t
    &\le \log\det(\vec G_t + \rho_{\max}\Eye)
    \le d\log(\rho_{\max} + tL^2/d).
  \end{align*}
  using the trace-determinant inequality as in the proof of
  \cref{lemma:elliptical-potential}.  All the $\beta_t$ are therefore
  bounded by the constants
  \begin{align*}
    \bar\beta_t &\defeq \sigma\sqrt{2\log\frac{2}{\alpha} + d\log\paren*{\frac{\rho_{\max}}{\rho_{\min}}
                 + \frac{tL^2}{d\rho_{\min}}}} + S\sqrt{\rho_{\max}} + \gamma. \qedhere
  \end{align*}
\end{proof}

We now take our first steps towards a regret bound by giving a
``generic'' version that depends only on LinUCB taking ``optimistic''
actions, the sizes of the confidence sets, and the rewards being
bounded.  We rely upon the upper bound for each $\beta_t$ shown in the
previous proposition.  We use the Cauchy-Schwarz inequality to bound
the sum of per-round regrets $r_t$ by $\sum r_t^2$; this results in
the leading $O(\sqrt n)$ factor in the regret bound.  Our
gap-dependent analysis later avoids this, but has other trade-offs.

\begin{lemma}[Generic LinUCB Regret]\label{lemma:linucb-regret}
  Suppose \cref{ass:meanreward-bound,ass:psd-reg} hold (i.e.\
  $\abs{\innerp{\vec\theta^*}{\vec x}} \le B$ and all $\vec H_t \succeq 0$)
  and $\bar\beta_n \ge \max\set{\beta_1,\dotsc,\beta_n,1}$; also assume
  that $B = 1$.  If all
  the confidence sets $\E_t$ contain $\vec\theta^*$ (i.e.,
  $\norm{\vec V_t}{\vec\theta^* - \tilde{\vec\theta}_t} \le
  \beta_t$), then the pseudo-regret of \cref{alg:linucb} is bounded by
  \begin{align*}
    \widehat{R}_n &\le \bar\beta_n\sqrt{4n \sum_{t=1}^n \min\set{1,
                   \norm{\inv{\vec V_t}}{\vec x_t}^2}}.
  \end{align*}
\end{lemma}

\begin{remark}[On the quantity $B$ appearing in
  \cref{ass:meanreward-bound}]%
  \label{remark:meanreward-bound}
  For the following proofs, we assume (as in this lemma) that
  \cref{ass:meanreward-bound} holds with $B=1$.  Eventually, however,
  our regret bounds end up with a factor $B$; we now explain how.
  First note that $B$ is trivially at most $LS$ by Cauchy-Schwarz:
  $\abs{\innerp{\vec\theta^*}{\vec x_t}} \le
  \norm{}{\vec\theta^*}\norm{}{\vec x_t} \le LS$.  The case where $B<1$ yet is some constant is trivial: clearly we can take $B=1$ without violating the assumption.  The case where $B=o(1)$ is actually quite intricate and somewhat ``unnatural'': while a-priori  we know the mean-reward can be as large as $LS$, it is in fact \emph{much} smaller. This means we have to scale down actions, and shrink the entire problem by a sub-constant; and as a result the noise $\sigma$ is actually now \emph{far} larger (it is like $\sigma/B$ in the original setting). While this can be a mere technicality in general, since our leading application is privacy this also means that the bounds on the actual reward we use in \cref{sec:alg-dp} needs to be scaled by a very large factor. Thus, allowing for ridiculously small $B$ turns into an unnecessary nuisance, and we simply assume that our upper-bound $B$ is not tiny~--- namely, we assume $B\geq 1$. 
  
  It
  remains to deal with the situation where $B>1$.
  In this case, we can pre-process the rewards to the algorithm,
  scaling them down by a factor of $B$.  If we also scale down all the
  actions $\vec x\in\D_t$, then the rest of the assumptions remain
  inviolate and the regret bound for $B=1$ applies.  However, this
  bounds the regret in the scaled problem: in the original problem the
  rewards and hence regret must be scaled up by a factor of $B$.  Note
  that by only scaling the $\D_t$, we are not modifying any
  quantities used in the actual algorithm, just the regret bound; this
  would not be true if we scaled $\vec\theta^*$ (whose maximum norm
  appears in the $\beta_t$), which is the other
  possibility to maintain the linearity of rewards.

  Indeed, in the scaled-down problem the regret is somewhat lower than
  the bound because both $L$ and $\sigma$ can be scaled down by $B$
  (the noise variance scales proportionally to the reward).  For
  simplicity, however, we refrain from replacing $L$ and $\sigma$ in
  the upper bound with $L/B$ and $\sigma/B$, respectively.
\end{remark}

\begin{proof}[Proof of \cref{lemma:linucb-regret}]
  At every round $t$, \cref{alg:linucb} selects an ``optimistic''
  action $\vec x_t$ satisfying
  \begin{align}\label{eq:optimistic-action}
    (\vec x_t,\bar{\vec\theta}_t)
    &\in \argmax_{(\vec x, \vec\theta) \in\D_t\times\E_t} \innerp{\vec\theta}{\vec x}.
  \end{align}
  Let $\vec x_t^* \in \argmax_{\vec x\in\D_t}\innerp{\vec\theta^*}{\vec x}$ be an optimal
  action and $r_t = \innerp{\vec\theta^*}{\vec x_t^* - \vec x_t}$ be the immediate
  pseudo-regret suffered for round $t$:
  \begin{align*}
    r_t &= \innerp{\vec\theta^*}{\vec x_t^*} - \innerp{\vec\theta^*}{\vec x_t} \\
        &\le \innerp{\bar{\vec\theta}_t}{\vec x_t} - \innerp{\vec\theta^*}{\vec x_t}
        &\text{from~\eqref{eq:optimistic-action} since }
          (\vec x_t^*,\vec\theta^*) \in \D_t\times\E_t \\
        &= \innerp{\bar{\vec\theta}_t - \vec\theta^*}{\vec x_t} \\
       &= \innerp{\vec V_t^{1/2}(\bar{\vec\theta}_t - \vec\theta^*)}{\vec V_t^{-1/2} \vec x_t}
        &\text{since } \vec V_t \succeq \vec H_t \succeq 0 \\
        &\le \norm{\vec V_t}{\bar{\vec\theta}_t - \vec\theta^*}\norm{\inv{\vec V_t}}{\vec x_t}
        &\text{by Cauchy-Schwarz}\\
        &\le \paren[\big]{\norm{\vec V_t}{\bar{\vec\theta}_t - \tilde{\vec\theta}_t}
          + \norm{\vec V_t}{\vec\theta^* - \tilde{\vec\theta}_t}}
          \norm{\inv{\vec V_t}}{\vec x_t}
        &\text{by the triangle inequality}\\
        &\le 2\beta_t\norm{\inv{\vec V_t}}{\vec x_t}
        &\text{since } \bar{\vec\theta}_t, \vec\theta^* \in \E_t \\
        &\le 2\bar\beta_n\norm{\inv{\vec V_t}}{\vec x_t}
        &\text{since } \bar\beta_n \ge \beta_t.
  \end{align*}
  From our assumptions that the mean absolute reward is at most $1$
  and $\bar\beta_n \ge 1$, we also get that $r_t \le 2 \le 2\bar\beta_n$.
  Putting these together,
  \begin{align}\label{eq:regret-bound-oneround}
    r_t &\le 2\bar\beta_n \min\set{1, \norm{\inv{\vec V_t}}{\vec x_t}}
  \end{align}
  Now we apply the Cauchy-Schwarz inequality, since $\widehat{R}_n =
  \innerp{\vec 1_n}{\vec r/n}$, where $\vec 1_n$ is the all-ones vector
  and $\vec r$ is the vector of per-round regrets:
  \begin{align*}
    \widehat R_n^2 &= n^2 \paren[\Big]{\sum_{t=1}^n \frac{r_t}{n}}^2
                   \le n^2 \sum_{t=1}^n \frac{r_t^2}{n} = n\sum_{t=1}^n {r_t}^2
                   \le 4n\bar\beta_n^2\sum_{t=1}^n\min\set{1,\norm{\inv{\vec V_t}}{\vec x_t}^2}.
  \end{align*}
  Taking square roots completes the proof.
\end{proof}

The following technical lemma relates the quantity from the previous
result to the volume (i.e.\ determinant) of the $\vec V_n$ matrix.
We will see shortly that the $\vec U_t$ are all lower bounds on the
$\vec V_t$.

\begin{lemma}[Elliptical Potential]\label{lemma:elliptical-potential}
  Let $\vec x_1,\dotsc,\vec x_n \in \Real^d$ be vectors with each
  $\norm{}{\vec x_t} \le L$.  Given a positive definite matrix
  $U_1\in\Real^{d\times d}$, define
  $U_{t+1} \defeq U_t + \vec x_t \transp{\vec x_t}$ for all $t$.  Then
  \begin{align*}
    \sum_{t=1}^n\min\set{1, \norm{\inv{U_t}}{\vec x_t}^2}
    &\le 2\log\frac{\det U_{n+1}}{\det U_1}
      \le 2d\log\frac{\tr U_1+nL^2}{d\det^{1/d} U_1}.
  \end{align*}
\end{lemma}

\begin{proof}
  We use the fact that $\min\set{1, u} \le 2\log(1+u)$ for any
  $u \ge 0$:
  \begin{align*}
    \sum_{t=1}^n\min\set{1, \norm{\inv{U_t}}{\vec x_t}^2}
    &\le 2 \sum_{t=1}^n \log(1 + \norm{\inv{U_t}}{\vec x_t}^2).
  \end{align*}
  We will show that this last summation is $2\log(\det U_{n+1}/\det
  U_n)$.  For all $t$, we have
  \begin{align*}
    U_{t+1} &= U_t + \vec x_t \transp{\vec x_t}
             = U_t^{1/2}
             \paren[\big]{I + U_t^{-1/2} \vec x_t \transp{\vec x_t} U_t^{-1/2}}
             U_t^{1/2} \\
    \det U_{t+1} &=\det U_t
                  \det\paren[\big]{I +
                  U_t^{-1/2} \vec x_t \transp{\vec x_t}
                  U_t^{-1/2}}.
  \end{align*}
  Consider the eigenvectors of the matrix $I + \vec y \transp{\vec y}$
  for an arbitrary vector $\vec y \in \Real^d$.  We know that $\vec y$
  itself is an eigenvector with eigenvalue $1+\norm{}{\vec y}^2$:
  \begin{align*}
    (I + \vec y \transp{\vec y}) \vec y
    &= \vec y + \vec y \innerp{\vec y}{\vec y} = (1+\norm{}{\vec y}^2)\vec y.
  \end{align*}
  Moreover, since $I + \vec y \transp{\vec y}$ is symmetric, every
  other eigenvector $\vec u$ is orthogonal to $\vec y$, so that
  \begin{align*}
    (I + \vec y \transp{\vec y}) \vec u
    &= \vec u + \vec u \innerp{\vec y}{\vec u} = \vec u.
  \end{align*}
  Therefore the only eigenvalues of $I + \vec y \transp{\vec y}$ are
  $1+\norm{}{\vec y}^2$ (with eigenvector $\vec y$) and 1.  In our
  case $\vec y = U_t^{-1/2} \vec x_t$ and
  $\norm{}{\vec y}^2 = \transp{\vec x_t} \inv{U_t} \vec x_t =
  \norm{\inv{U_t}}{\vec x_t}^2$, so we get our first inequality:
  \begin{align*}
    \det U_{n+1} &= \det U_1 \prod_{t=1}^n(1 + \norm{\inv{U_t}}{\vec x_t}^2) \\
    2\log\frac{\det U_{n+1}}{\det U_1}
            &= 2\sum_{t=1}^n\log(1+\norm{\inv{U_t}}{\vec x_t}^2).
  \end{align*}
  To get the second inequality, we apply the arithmetic-geometric
  mean inequality to the eigenvalues $\lambda_i$ of $U_n$:
  \begin{align*}
    \det U_n &= \prod_{i=1}^d \lambda_i
              \le \paren[\Big]{\frac{1}{d} \sum_{i=1}^d \lambda_i}^d
              = \paren{(1/d)\tr U_n}^d
              \le \paren{(\tr U_1 + nL^2)/d}^d \\
    2\log\frac{\det U_n}{\det U_1}
            &\le 2d \log\frac{\tr U_1 + nL^2}{d\det^{1/d}U_1}
              \qedhere
  \end{align*}
\end{proof}

We are finally in a position to prove the main regret theorem.  The
proof is straightforward and essentially comes down to plugging in our
preceding results.

\ThmLinUCBRegret*

\begin{proof}
  We restrict ourselves to the event that all the confidence
  ellipsoids contain $\vec\theta^*$ and all
  $\rho_{\min}\Eye \preceq \vec H_t \preceq \rho_{\max}\Eye$.
  \Cref{prop:calc-beta} assures us that this happens with probability at least
  $1-\alpha$, and furthermore gives us the bound $\beta_t \le \bar\beta_n$:
  \begin{align*}
    \bar\beta_n &\defeq \sigma\sqrt{2\log\frac{2}{\alpha} + d\log\paren*{\frac{\rho_{\max}}{\rho_{\min}}
                 + \frac{nL^2}{d\rho_{\min}}}}
                 + S\sqrt{\rho_{\max}} + \gamma.
  \end{align*}
  Next, we have
  $\norm{\inv{\vec V_t}}{\vec x_t} \le
  \norm{\inv{(\vec G_t+\rho_{\min}\Eye)}}{\vec x_t}$, which applied to the
  result of \cref{lemma:linucb-regret} gives, using
  \cref{lemma:elliptical-potential}
  \begin{align*}
    \widehat{R}_n
    &\le \bar\beta_n\sqrt{8dn\log\paren*{1+\frac{nL^2}{d\rho_{\min}}}} \\
    &\le \sqrt{8n}\brck*{\sigma\paren*{2\log\frac{2}{\alpha}
      + d\log\paren*{\frac{\rho_{\max}}{\rho_{\min}} + \frac{nL^2}{d\rho_{\min}}}}
      + (S\sqrt{\rho_{\max}}+\gamma)\sqrt{d\log\paren*{1+\frac{nL^2}{d\rho_{\min}}}}}.
  \end{align*}
  The argument outlined in \cref{remark:meanreward-bound} preceding
  the proof of \cref{lemma:linucb-regret} tells us how to reintroduce
  the missing factor of $B$ in this regret bound.
\end{proof}

The proof of the gap-dependent regret bound diverges from the previous
proof in only one major way: the gap is used to bound each $r_t$ by
$r_t^2/\Delta$.  Then the sum of $r_t^2$ is bounded as before; this
avoids the $\sqrt n$ factor introduced by the use of the
Cauchy-Schwarz inequality.

\ThmLinUCBGapRegret*

\begin{proof}
  Because of the gap assumption, for every round $t$ if the per-round
  pseudo-regret $r_t \neq 0$ then $r_t \ge \Delta$.  We use this fact
  to decompose the regret in a different way than we did in
  \cref{lemma:linucb-regret}.  The rest of the proof is similar to
  that of \cref{thm:linucb-regret}.  As before, see
  \cref{remark:meanreward-bound} preceding the proof of
  \cref{lemma:linucb-regret} to introduce the missing $B$ factor.
  \begin{align*}
    \widehat R_n
    &= \sum_{t\in B_n} r_t \le \sum_{t\in B_n} \frac{r_t^2}{\Delta} \\
    &\le \frac{4}{\Delta} \bar\beta_n^2 \sum_{t\in B_n} \min\set{1, \norm{\inv{\vec V_t}}{\vec x_t}^2}
    &\text{from~\eqref{eq:regret-bound-oneround}} \\
    &\le \frac{8}{\Delta} \bar\beta_n^2d\log\paren*{1+\frac{nL^2}{d\rho_{\min}}} \\
    &\le \frac{8}{\Delta} \brck*{\sigma\paren*{2\log\frac{2}{\alpha}
      + d\log\paren*{\frac{\rho_{\max}}{\rho_{\min}} + \frac{nL^2}{d\rho_{\min}}}}
      + (S\sqrt{\rho_{\max}}+\gamma)\sqrt{d\log\paren*{1+\frac{nL^2}{d\rho_{\min}}}}}^2
    &&\qedhere
  \end{align*}
\end{proof}

\subsection{Regret Bounds Open Problem}%
\label[subsection]{sec:regret-discussion}

The first conclusion from these regret bounds is that allowing
changing regularizers does not incur significant additional regret, as
long as they are bounded both above and below.  Broadly speaking,
these bounds for contextual linear bandits match those for standard
MAB algorithms in terms of their dependence on $n$ and $\Delta$ ---
just like with UCB, for example, the minimax bound is $O(\sqrt n)$ and
the gap-dependent bound is $O(\log(n)/\Delta)$.  However, the
dependence on $d$ (which corresponds to the number of arms for the
MAB) is much worse, with $O(d)$ in the minimax case and $O(d^2)$ in
the gap-dependent case.

It is an interesting open question whether the $O(d^2)$ dependence on
$d$ is necessary to achieve $O(\log n)$ gap-dependent regret bounds.
As we were unable to prove a lower bound of $\Omega(d^2)$, we resorted
to empirically checking the performance of the (non-private) LinUCB on
such $\Delta$-gap instances; the results can be found in
\cref{sec:expt-regret-vs-d}.

\section{Privacy Proofs}%
\label[section]{apx_sec:privacy_proofs}

We now provide the missing privacy
proofs from the main body of the paper.  First, we give the omitted
proof from \cref{sec:dp-wishart}.

\PropWishartTails*
\begin{proof}
  Seeing as $\vec H_t\sim\Wishart_d(\tilde L^2 \Eye, mk)$,
  straight-forward bounds on the eigenvalues of the Wishart
  distribution (e.g.~\cite{SheffetPrivateApproxRegression2015},
  Lemma~A.3) give that w.p.\ $\geq 1- \alpha/2n$ all of the
  eigenvalues of $\vec H_t$ lie in the interval
  $\tilde L^2 \paren[\big]{\sqrt{mk} \pm \paren[\big]{\sqrt{d} +
      \sqrt{2\ln(8n/\alpha)}}}^2$. To bound
  $\norm{\inv{\vec H_t}}{\vec h_t}$ we draw back to the definition of
  the Wishart distribution as the Gram matrix of samples from a
  multivariate Gaussian $\Normal(\vec 0, \tilde L^2\Eye)$. Denote this
  matrix of Gaussians as $[\vec Z ; \vec z]$ where
  $\vec Z\in \Real^{mk\times d}$ and $\vec z \in \Real^{mk}$, and we
  have that $\vec H_t = \transp{\vec Z} \vec Z$ and
  $\vec h_t = \transp{\vec Z} \vec z$, thus
  $\norm{\inv{\vec H_t}}{\vec h_t} = \sqrt{ \transp{\vec z} \vec Z
    \inv{(\XtX{\vec Z})} \transp{\vec Z} \vec z }$. The
  matrix $\vec Z \inv{(\XtX{\vec Z})} \transp{\vec Z}$ is a
  projection matrix onto a $d$-dimensional space, and projecting the
  spherical Gaussian $\vec z$ onto this subspace results in a
  $d$-dimensional spherical Gaussian. Using concentration bounds on
  the $\chi^2$-distribution (\cref{claim:chi2-tails}) we have that
  w.p.\ $\geq 1- \alpha/2n$ it holds that
  $\norm{\inv{\vec H_t}}{\vec h_t} \le \gamma \defeq \tilde L\paren[\big]{\sqrt{d} +
    \sqrt{2\ln(2n/\alpha)}}$.
  
  It is straightforward to modify these bounds for the shifted
  regularizer matrix $\vec H_t' \defeq \vec H_t - c\Eye$; the minimum
  and maximum eigenvalues are bounded as
  $\rho_{\min}' = \rho_{\min} - c$ and
  $\rho_{\max}' = \rho_{\max} - c$, respectively.  The value of $c$ in
  \cref{eq:c} is chosen so that
  $\rho_{\min}' = \rho_{\min}-c = \rho_{\max} - \rho_{\min} = 4\tilde
  L^2\sqrt{mk}(\sqrt d + \sqrt{2\ln(\nicefrac{8n}{\alpha})})$.  It
  follows that
  $\rho_{\max}' = \rho_{\max} - c = \rho_{\min}' + \rho_{\max} -
  \rho_{\min} = 2\rho_{\min}'$.  Finally,
  \cref{claim:shifted-matrix-norm} gives
  \begin{align*}
    \norm{\inv{\vec H_t'}}{\vec h_t}
    &\le \norm{\inv{\vec H_t}}{\vec
    h_t}\sqrt{\rho_{\min}/\rho_{\min}'}
    \le \gamma \sqrt{\rho_{\min}/\rho_{\min}'} \\
    &= \tilde L\paren[\big]{\sqrt{d} + \sqrt{2\ln(2n/\alpha)}}
      \sqrt{\frac{
      \tilde L^2\paren[\big]{\sqrt{mk} - \sqrt{d} - \sqrt{2\ln(8n/\alpha)}}^2
      }{
      4\tilde L^2\sqrt{mk}\paren[\big]{\sqrt d + \sqrt{2\ln(8n/\alpha)}}
      }}\\
    &\leq \tilde L \frac{\paren[\big]{\sqrt{mk} - \sqrt{d} - \sqrt{2\ln(8n/\alpha)}}\sqrt{\sqrt{d} + \sqrt{2\ln(2n/\alpha)}}} {4(mk)^{1/4}} \\
    &\leq \tilde L \sqrt{\sqrt{mk} \paren[\big]{\sqrt{d} + \sqrt{2\ln(2n/\alpha)}}} \eqdef \gamma'
      \qedhere
  \end{align*}
\end{proof}

\begin{theorem}[{\citealp[Theorem~4.1]{SheffetPrivateApproxRegression2015}}]%
  \label{thm:wishart-dp}%
  Fix $\varepsilon\in(0,1)$ and $\delta\in(0,1/e)$.  Let
  $A\in\Real^{n\times p}$ be a matrix whose rows have $l_2$-norm
  bounded by $\tilde L$.  Let $W$ be a matrix sampled from the
  $d$-dimensional Wishart distribution with $k$ degrees of freedom
  using the scale matrix $\tilde L^2\Eye{p}$ (i.e.\
  $W \sim \Wishart_p(\tilde L^2\Eye{p}, k)$) for
  $k \ge p + \floor[\big]{\frac{14}{\varepsilon^2}\cdot 2\log(4/\delta)}$.
  Then outputting $\XtX{A} + N$ is
  $(\varepsilon,\delta)$-differentially private with respect to
  changing a single row of $A$.
\end{theorem}

We now give the proof of the lower bound of any private algorithm
under the standard notion of differential privacy under continual
observation, as discussed in \cref{sec:lower-bounds}. First, of
course, we need to define this notion. \DPDefinition{} We now prove
the following.

\clmLinearRegretDP*

\DPLowerBoundProof

\section{Experiments}
\label[section]{sec:experiments}

We performed some experiments with synthetic data to characterize the
performance of the algorithms in this paper.

\paragraph{Setting.} We first describe the
common setting used for all the experiments:
Given a dimension $d$, we first select $\vec\theta^*$ to be a random
unit vector in $\Real^d$ (distributed uniformly on the hyper-sphere,
so that $S = 1$).  Then we construct decision sets of size $K$
($K=d^2$ in our experiments), consisting of one \emph{optimal
  action} and $K-1$ \emph{suboptimal actions}, all of unit length (so
$L = 1$).  The optimal action is chosen uniformly at random from the
$(d-2)$-dimensional set
$\set{\vec x\in\Real^d\given \norm{}{\vec x} = 1, \innerp{\vec x}{\vec
    \theta^*} = 0.75}$.  The suboptimal actions are chosen
independently and uniformly at random from the $(d-1)$-dimensional set
$\set{\vec x\in\Real^d\given \norm{}{\vec x} = 1, \innerp{\vec
    x}{\vec\theta^*} \in [-0.75, 0.65]}$.  This results in a
suboptimality gap of $\Delta=0.1$, since the optimal arm has mean
reward $0.75$ and the suboptimal arms have mean rewards in the
$[-0.75, 0.65]$ interval.  To simulate the contextual bandit setting,
a new decision set is sampled before each round.
\begin{figure}[h]
  \centering
  \includegraphics[width=0.5\columnwidth]{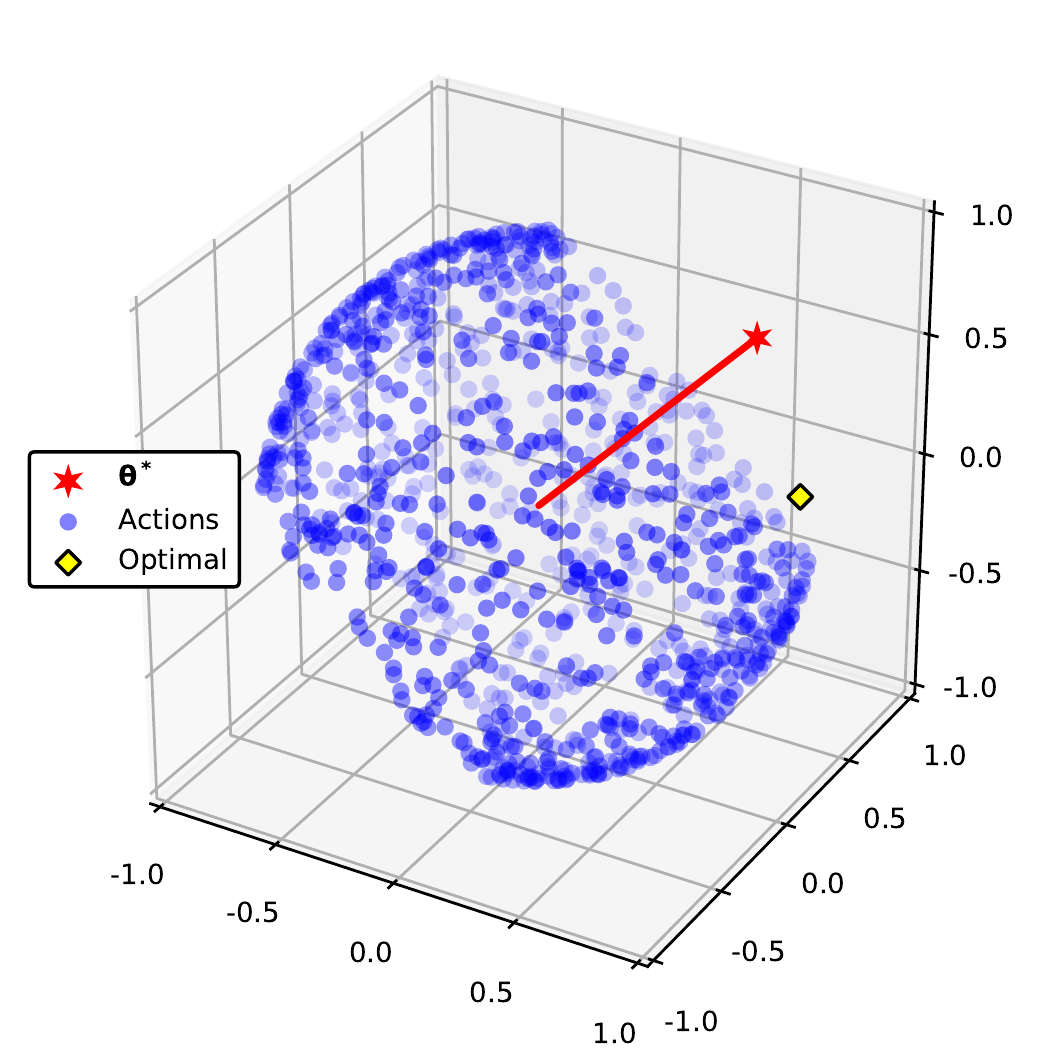}
  \caption{An example decision set in $\Real^3$ with $K=1000$ actions.}
  \label{fig:actions}
\end{figure}

The rewards are either $-1$ or $+1$, with probabilities chosen so that
$\Ex{y_t} = \innerp{\vec x_t}{\vec\theta^*}$.  Therefore $\tilde B=1$
and, being bounded in the $[-1, 1]$ interval, the reward distribution
is subgaussian with $\sigma^2=1$.

The experiments below measure the expected regret in each case; the
confidence parameter is $\alpha=1/n$, which is the usual choice when
one wishes to minimize expected regret.

\subsection{The Dependency of the Pseudo-Regret on the Dimension for Gap Instances}
\label[subsection]{sec:expt-regret-vs-d}

The first experiment was aimed at the open question of
\cref{sec:regret-discussion}, namely whether the gap-dependent regret
is $\Omega(d^2)$ in the dimension of the problem. Thus privacy wasn't
a concern in this particular setting; rather, our goal was to
determine the performance of our general recipe algorithm in a
contextual setting with a clear-cut gap.  We measured the
pseudo-regret of the non-private LinUCB algorithm as a function of the
dimension over $n=10^5$ rounds with the regularizer
$\rho=\vec I_{d\times d}$ and $K=d^2$ arms.  The values of $d$ were
logarithmically spaced in the interval $[4, 64]$.  The results of the
experiment are plotted in \cref{fig:expt1-regret-vs-time}.  The two
sub-experiments differ only in the reward noise distribution used.  In
the first, the reward noise is truly a Gaussian with $\sigma^2=1$,
whereas in the second the reward is $\pm 1$ as described above
(subgaussian with $\sigma^2=1$).  In the latter case, the actual
variance in the reward depends on its expectation, and is somewhat
lower than 1.  This is perhaps why the regret is somewhat lower than
with gaussian reward noise.

\begin{figure}
  \centering
\includegraphics[width=\columnwidth]{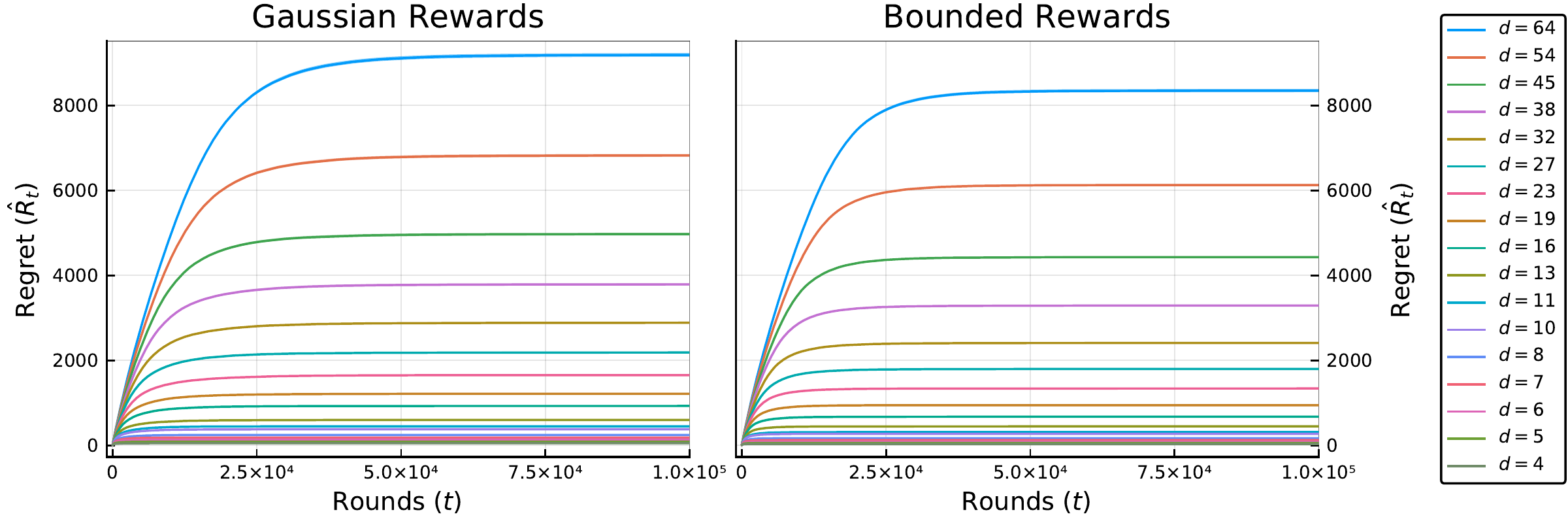}
  \caption{Experiment 1 --- Regret over time for varying dimensions.}%
  \label{fig:expt1-regret-vs-time}
\end{figure}

\Cref{fig:expt1-regret-vs-d} shows the same results with total
accumulated regret plotted against dimension using a log--log scale.
The best-fit line on this plot has a slope of roughly 2, clearly
pointing to a super-linear dependency on $d$.  We conjecture that, in
general, the dependency on $d$ is indeed quadratic.

\begin{figure}
  \centering
  \includegraphics[width=0.5\columnwidth]{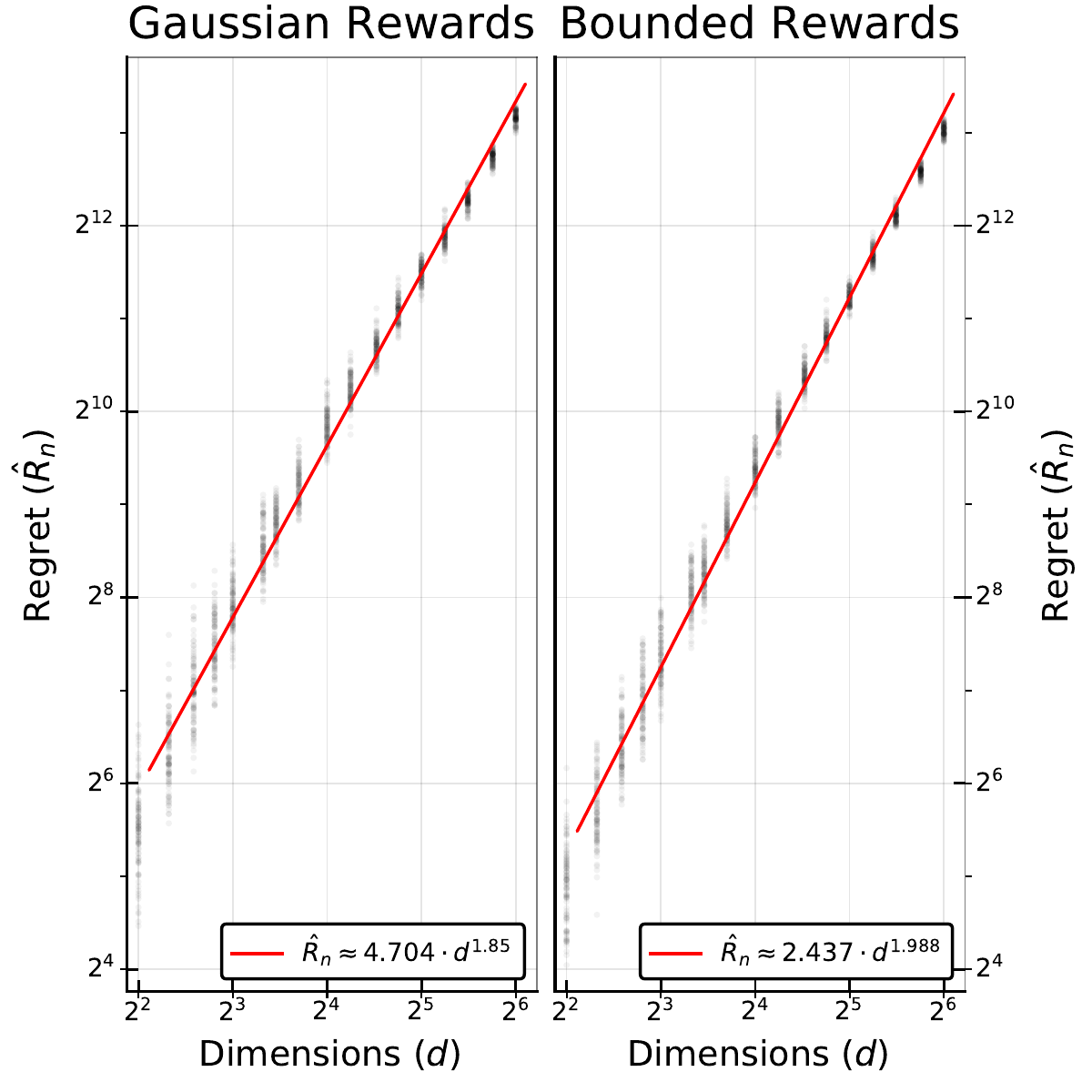}
  \caption{Experiment 1 --- Regret vs.\ dimension with log--log axes
    and best-fit line.}\label{fig:expt1-regret-vs-d}
\end{figure}

\subsection{Empirical Performance of the Privacy-Preserving Algorithms over Time}
\label{sec:expt-regret-time}

This experiment compares the expected regret of the various algorithm
variants presented in this paper.  The two major privacy-preserving
algorithms are based on Wishart noise (\cref{sec:dp-wishart}) and
Gaussian noise (\cref{sec:dp-gauss}); both were run with privacy
parameters $\varepsilon = 1.0$ and $\delta = 0.1$ over a horizon of
$n = 5\times 10^7$ rounds and dimension $d=5$ and $K=d^2=25$.  The
results are shown in \cref{fig:expt2-regret-time}; the curves are
truncated after $2\times 10^7$ rounds because they are essentially
flat after this point.

The sub-figures of \cref{fig:expt2-regret-time} show two settings that
differ in the sub-optimality gap $\Delta$ between the rewards of the
optimal and sub-optimal arms.  In the left sub-figure, the algorithms
are run in a setting without a structured gap ($\Delta = 0$), where we
have not forced all arms to be strictly separated from the optimal arm
by a large reward gap.  Here, all sub-optimal arms are distributed
uniformly on the set
$\set{\vec x\in\Real^d\given \norm{}{\vec x} = 1, \innerp{\vec
    x}{\vec\theta^*} \in [-0.75, 0.75]}$ (and \emph{not} from
$[-0.75,0.65]$ as in the previous experiment).  Note that while we
cannot guarantee that in \emph{all} rounds there exists a gap between
the optimal and sub-optimal arms, it is still true that \emph{in
  expectation} we should observe a gap of $\Theta(\nicefrac 1 K)$
between the optimal arm and the second-best arm (and as $K=25$ this
expected gap is, still, a constant in comparison to $n$).  In the
right sub-figure, however, the sub-optimal arms are indeed separated
by a gap of $\Delta=0.1$ from the optimal arms; their rewards lie in
the interval $[-0.75, 0.65]$ as in the previous experiment.  In both
cases there is always an optimal arm with reward $0.75$.

\begin{figure}
  \centering
  \includegraphics[width=\columnwidth]{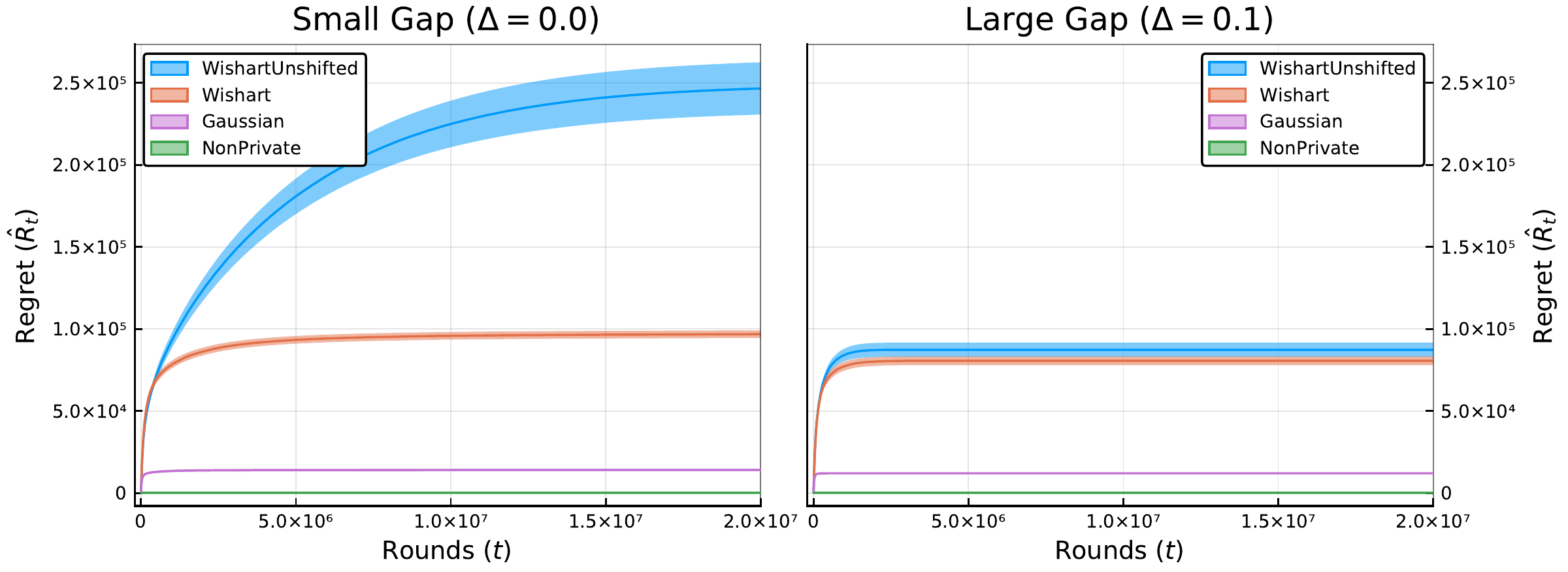}
  \caption{Experiment 2 --- Regret over time, with and without forced
    sub-optimality gaps.}
  \label{fig:expt2-regret-time}
\end{figure}

The figures show the following algorithm variants:
\begin{itemize}
\item The \texttt{NonPrivate} algorithm is LinUCB with regularizer
  $\rho=1.0$.  Its regret is too small to be distinguished from the
  x-axis in this plot.
\item The \texttt{Gaussian} variant is described in
  \cref{sec:dp-gauss}.
\item The \texttt{Wishart} variant is described in
  \cref{sec:dp-wishart} with the shift given in \cref{eq:c}.
\item The \texttt{WishartUnshifted} variant is that of
  \cref{sec:dp-wishart} but with no shift.
\end{itemize}

\paragraph{Results.}

It is apparent that, at least for this setting, the Gaussian noise
algorithm outperforms Wishart noise.  This shows that while the
asymptotic performance of the two algorithms is fairly close, the
constants in the Gaussian version of the algorithm are far better than
the ones in the Wishart-noise based algorithm.

Furthermore, the performance of the \texttt{WishartUnshifted} variant
changes significantly between the two cases --- it has the worst
regret in the no-gap setting ($\Delta=0$) but, surprisingly, it is
statistically indistinguishable from the shifted \texttt{Wishart} variant
in the large gap instance ($\Delta=0.1$).  We investigate this
relationship between the sub-optimality gap and shifted regularizers
in the next experiment.


\subsection{Empirical Performance of Shifted Regularizers for
  Different Suboptimality Gaps}
\label{sec:expt-gap-shift}

In both the Wishart and Gaussian variants of our algorithm, we use a
\emph{shifted} regularization matrix $\vec H_t \pm c\Eye$, choosing
the shift parameter $c$ to approximately optimize our regret bound in
each case.  This optimal shift parameter turns out not to depend on
the sub-optimality gap $\Delta$ of the problem instance.  The previous
experiment showed, however, that in practice the relative performance
of the shifted and unshifted Wishart variants changes drastically
depending on the gap.  In this experiment, we investigate the impact
of varying the shift parameter for the Wishart and Gaussian mechanisms
under different sub-optimality gaps $\Delta$.

All the parameters are the same as the previous experiment --- the
only difference is the shift parameter; the results are shown in
\cref{fig:expt3-gap-shift}.  The two sub-figures show the performance
of the Wishart and Gaussian variants, respectively.  The x-axis is a
logarithmic scale indicating $\rho_{\min}$, the high-probability lower
bound on the minimum eigenvalue of the shifted regularizer matrix.
$\rho_{\min}$ serves as a good proxy for the shift parameter because
changing one has the effect of shifting the other by the same amount;
it has the added benefit of being meaningfully comparable amongst the
different algorithm variants.  The vertical dotted lines indicate the
$\rho_{\min}$ values corresponding to the algorithms from
\cref{sec:dp-wishart,sec:dp-gauss} for the Wishart and Gaussian
variants, respectively; these are also the algorithms examined in the
previous experiment.  The Gaussian mechanism does not have an
unshifted variant.

\paragraph{Results.}

Tuning the shift parameter appears to significantly affect the
performance only for problem instances with relatively small or zero
sub-optimality gaps.  In the large-gap settings, on the other hand,
having too much regularization does not seem to increase regret
appreciably.  The small-gap settings are exactly those in which
exploration is crucial, so we conjecture that large regularizers
inhibit exploration and thereby incur increased regret.

\begin{figure}
  \centering
  \includegraphics[width=\columnwidth]{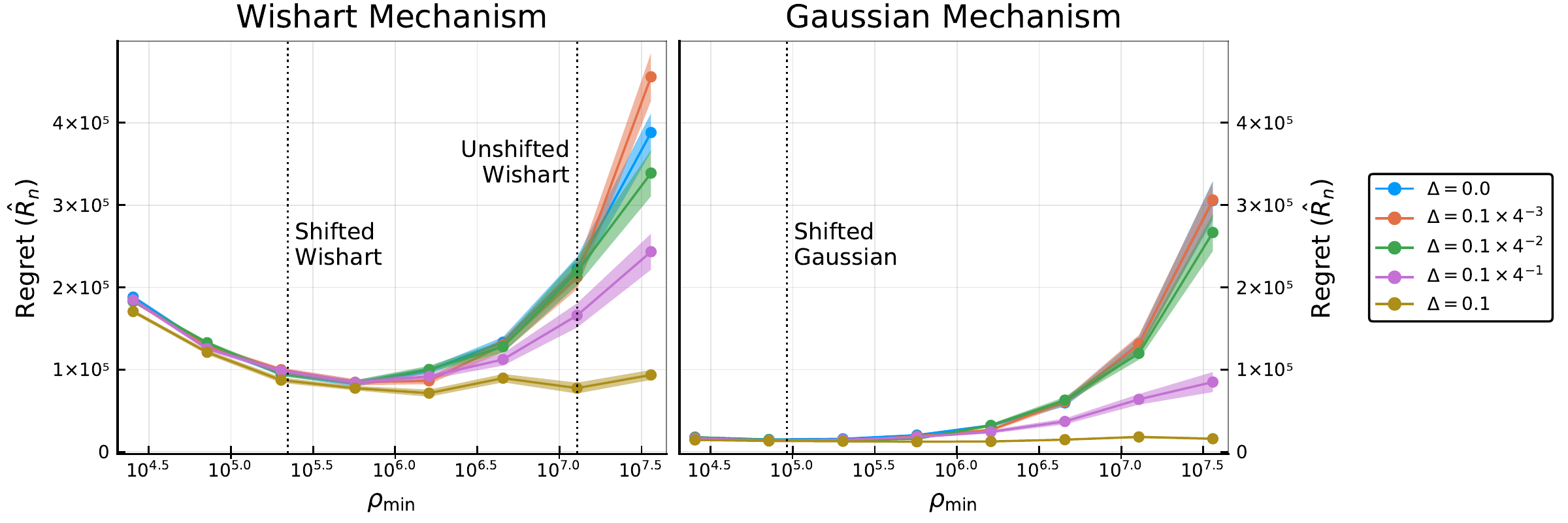}
  \caption{Experiment 3 --- Varying shift parameters with
    different sub-optimality gaps.}%
  \label{fig:expt3-gap-shift}
\end{figure}

\end{document}